\documentclass[11pt]{amsart}
\usepackage{fullpage}
\usepackage{lipsum}                     
\usepackage{xargs}
\usepackage{xfrac}
\usepackage{float}
\usepackage{hyperref} 
\usepackage{lipsum}                     
\usepackage{xargs}
\usepackage{xfrac}
\usepackage{amssymb}
\usepackage{amsmath}

\usepackage{mathtools}
\usepackage{algorithm2e}
\usepackage{dsfont}

\newcommand{\tr}{\text{Tr}}

\newcommand{\Normal}{\mathcal{N}}
\def\1{\mathds{1}}

\newcommand{\law}{\mathrm{law}}

\newcommand{\bigo}{\mathcal{O}}

\usepackage{amsmath,amssymb}

\newcommand{\ones}{{\bf 1}}

\newcommand{\W}{{\mathcal W}}
\newcommand{\E}{{\mathbb E}}

\newcommand{\N}{\mathcal{N}}
\newcommand{\R}{{\mathbb{R}}}

\newcommand{\diag}{\text{diag}}

\usepackage{mdframed}
\usepackage{algorithmicx}

\usepackage{eqparbox}
\newcommand{\comment}[1]{}

\usepackage{amsthm}

\newtheorem{theorem}{Theorem}
\newtheorem{lemma}[theorem]{Lemma}
\newtheorem{corollary}[theorem]{Corollary}

\newtheorem*{assumption*}{\assumptionnumber}
\providecommand{\assumptionnumber}{}
\makeatletter
\newenvironment{assumptioncost}[2]
 {%
  \renewcommand{\assumptionnumber}{Assumption $\mathcal{A}_#1(#2)$}%
  \begin{assumption*}%
  \protected@edef\@currentlabel{$\mathcal{A}_#1$}%
 }
 {%
  \end{assumption*}
 }
 \makeatother

\usepackage{thmtools,thm-restate}

\usepackage{eqparbox}
\usepackage[english]{babel}
\usepackage{natbib}
\usepackage{caption}
\usepackage{subcaption}
\usepackage{amsmath,amsthm, amssymb, bbm,  color}
\usepackage{mathbbol}
\usepackage{graphicx}
\usepackage{fullpage}
\usepackage{amsmath,tkz-euclide}
\usepackage{amsmath}
\usepackage{enumitem}
\usepackage{tikz}

\let\oldmarginpar\marginpar
\renewcommand\marginpar[1]{\-\oldmarginpar[\raggedleft\scriptsize #1]%
{\raggedright\scriptsize #1}}

\title{Batch Normalization Orthogonalizes Representations in Deep Random Networks}

 \author{Hadi Daneshmand \and Amir Joudaki \and Francis Bach }
\thanks{INRIA Paris, ETH Zurich, INRIA-ENS-PSL Paris.}
\begin{document}


\maketitle

\begin{abstract}
This paper underlines a subtle property of batch-normalization (BN): Successive batch normalizations with random linear transformations make hidden representations increasingly orthogonal across layers of a deep neural network. We establish a non-asymptotic characterization of the interplay between depth, width, and the orthogonality of deep representations. More precisely, under a mild assumption, we prove that the deviation of the representations from orthogonality rapidly decays with depth up to a term inversely proportional to the network width. This result has two main implications: 1) Theoretically, as the depth grows, the distribution of the representation --after the linear layers-- contracts to a Wasserstein-2 ball around an isotropic Gaussian distribution. Furthermore, the radius of this Wasserstein ball shrinks with the width of the network. 2) In practice, the orthogonality of the representations directly influences the performance of stochastic gradient descent (SGD). When representations are initially aligned, we observe SGD wastes many iterations to orthogonalize representations before the classification. Nevertheless, we experimentally show that starting optimization from orthogonal representations is sufficient to accelerate SGD, with no need for BN.
\end{abstract}

\section{Introduction}

Batch Normalization (BN)~\citep{ioffe2015batch} enhances training across a wide range of deep network architectures and experimental setups~\citep{he2016deep,huang2017densely,silver2017mastering}. The practical success of BN has inspired research into the underlying benefits of BN~\citep{santurkar2018does,karakida2019normalization,arora2018theoretical,bjorck2018understanding}. In particular, it shown that BN influences first-order optimization methods by avoiding the rank collapse in deep representation \citep{daneshmand2020batch},  direction-length decoupling of optimization \citep{kohler2018exponential}, influencing the learning rate of the steepest descent  \citep{arora2018theoretical,bjorck2018understanding}, and smoothing the optimization objective function \citep{santurkar2018does,karakida2019normalization}. However, the benefits of BN go beyond its critical role in optimization. For example, \cite{frankle2020training} shows that BN networks with random weights also achieve surprisingly high performance after only minor adjustments of their weights. This striking result motivates us to study the representational power of random networks with BN.

We study hidden representations across layers of a laboratory random BN with linear activations. Consider a batch of samples passing through consecutive BN and linear layers with Gaussian weights. The representations of these samples are perturbed by each random linear transformation, followed by a non-linear BN. At first glance, the deep representations appear unpredictable after many stochastic and non-linear transformations. Yet, we show that these transformations orthogonalize the representations. To prove this statement, we introduce the notion of  ``orthogonality gap'', defined in Section \ref{sec:orthogonality}, to quantify the deviation of representations from a perfectly orthogonal representation. Then, we prove that the orthogonality gap decays exponentially with the network depth and stabilizes around a term inversely related to the network width. More precisely, we prove 
\begin{align} \nonumber
     \E\Bigg[ \text{orthogonality gap} \Bigg] = \bigo\left( \left(1-\alpha\right)^{\text{depth}}+ \frac{\text{batch size}}{\alpha\sqrt{\text{width}}}\right)
\end{align}
holds for $\alpha>0$ that is an absolute constant under a mild assumption. In probability theoretic terms, we prove stochastic stability of the Markov chain of hidden representations \citep{kushner1967stochastic,kushner2003stochastic,khasminskii2011stochastic}. The orthogonality of deep representations allows us to prove that the distribution of the representations after linear layers contracts to a Wasserstein-2 ball around isotropic Gaussian distribution as the network depth grows. Moreover, the radius of the ball is inversely proportional to the network width. Omitting details, we prove the following bound holds:
\begin{align} \nonumber
    \text{Wasserstein}_2(\text{representations},\text{Gaussian})^2= \bigo\left( \left(1-\alpha\right)^{\text{depth}}\left(\text{batch size}\right)+ \frac{ \text{batch size}}{\alpha\sqrt{\text{width}}}\right).
\end{align}
The above equation shows how depth,  width, and batch size, interact with the law of the representations. Since the established rate is exponential with depth, the distribution of the representations stays in a Wasserstein ball around isotropic Gaussian distribution after a few layers. Thus, BN not only stabilizes the distribution of the representations, which is its main promise~\citep{ioffe2015batch}, but also enforces Gaussian isotropic distribution in deep layers. 

There is growing interest in bridging the gap between neural networks, as the most successful parametric methods for learning, and Gaussian processes and kernel methods, as well-understood classical models for learning \citep{jacot2018neural,matthews2018gaussian,lee2017deep,bietti2019inductive,huang2014kernel}. This link is inspired by studying random neural networks in the asymptotic regime of infinite width. The seminal work by \cite{neal2012bayesian} sparks a single-layer network resembles a Gaussian process as its width goes to infinity. However, increasing the depth may significantly shift the distribution of the representations away from Gaussian \citep{ioffe2015batch}. This distributional shift breaks the link between Gaussian processes and deep neural networks. To ensure Gaussian representations,  \cite{matthews2018gaussian} suggests increasing the width of the network proportional to the network depth. Here, we show that BN ensures Gaussian representations even for deep networks with \textit{finite} width. This result bridges the link between deep neural networks and Gaussian process in the regime of finite width. Many studies rely on deep Gaussian representations in an infinite width setting \citep{yang2019mean,schoenholz2016deep,pennington2018emergence,klambauer2017self,de2018random}. Our non-asymptotic Gaussian approximation can be incorporated into their analysis to extend these results to the regime of finite width. 

Since training starts from random networks,  representations in these networks directly influence  training. Hence, recent  theoretical studies has investigated the interplay between initial hidden representations and training~\citep{daneshmand2020batch,bjorck2018understanding,frankle2020training,schoenholz2016deep,saxe2013exact,bahri2020statistical}.
In particular, it is shown that hidden representations in random networks \emph{without BN} become correlated as the network grows in depth, thereby drastically slows training~\citep{daneshmand2020batch,he2016deep,bjorck2018understanding,saxe2013exact}. 
On the contrary,  we prove that deep representations in networks \emph{with BN} are almost orthogonal.
We experimentally validate that initial orthogonal representations can save training time that would otherwise be needed to orthogonalize them. By proposing a novel initialization scheme, we ensures the orthogonality of hidden representations; then, we show that such an initialization effectively avoids the training slowdown with depth for vanilla networks, with no need for BN.  This observation further motivates studying the inner workings of BN to replace or improve it in deep neural networks. \\
\newpage
\noindent Theoretically, we made the following contributions:
\begin{enumerate}
    \item  For MLPs with batch normalization, linear activation, and Gaussian i.i.d. weights, we prove that representations across layers become increasingly orthogonal up to a constant inversely proportional to the network width. 
    \item Leveraging the orthogonality, we prove that the distribution of the representations contracts to a Wasserstein ball around a Gaussian distribution as the depth grows. Up to the best of our knowledge,  this is the first \textit{non-asymptotic} Gaussian approximation for deep neural networks with finite width.
\end{enumerate}
Experimentally, we made the following contribution\footnote{The codes for the experiments are available at \href{https://github.com/hadidaneshmand/batchnorm21.git}{https://github.com/hadidaneshmand/batchnorm21.git}}: 
\begin{enumerate}[resume]
    \item  Inspired by our theoretical understanding, we propose a novel weight initialization for standard neural networks that ensure orthogonal representations without BN. Experimentally, we show that this initialization effectively avoids training slowdown with depth in the absence of BN. 
\end{enumerate}

\section{Preliminaries} \label{sec:perliminaries}
\subsection{Notations}
Akin to~\cite{daneshmand2020batch}, we focus on a Multi-Layer Perceptron (MLP) with batch normalization and linear activation. Theoretical studies of linear networks is a growing research area~\citep{saxe2013exact,daneshmand2020batch,bartlett2019gradient,arora2018convergence}.  When weights initialized randomly, linear and non-linear networks share similar properties~\citep{daneshmand2020batch}. For ease of analysis, we assume activations are linear. Yet, we will argue that our findings extend to non-linear in Appendix~\ref{sec:activations}.

We use $n$ to denote batch size, and $d$ to denote the width across all layers, which we further assume is larger than the batch size $d\geq n$. Let $H_\ell \in \R^{d\times n}$ denote $n$ sample representations in layer $\ell$, with $H_0 \in \R^{d \times n}$ corresponding to $n$ input samples in the batch with $d$ features. Successive representations are connected by Gaussian weight matrices $W_\ell\sim\Normal(0,I_d/d)$ as 
\begin{align}
    \label{eq:chain_linear}
    H_{\ell+1} = \frac{1}{\sqrt{d}} BN(W_{\ell} H_{\ell}), \quad BN(M) = \diag(M M^\top)^{-\sfrac{1}{2}} M,
\end{align}
where $\diag(M)$ zeros out off-diagonal elements of its input matrix, and the scaling factor $1/\sqrt{d}$ ensures that all matrices $\{H_k\}_{k}$ have unit Frobenius norm (see Appendix \ref{sec:perliminaries}). The BN function in Eq.~(\ref{eq:chain_linear}) differs slightly from the commonly used definition for BN as the mean correction is omitted. However, \cite{daneshmand2020batch} showes this difference does not change the network properties qualitatively. Readers can find all the notations in Appendix Table~\ref{tab:notations}.

\subsection{The linear independence of hidden representations}
\cite{daneshmand2020batch} observe that if inputs are linearly independent, then their hidden representations remain linearly independent in all layers as long as $d = \Omega(n^2)$. Under technical assumptions, \cite{daneshmand2020batch} establishes a lower-bound on the average of the rank of hidden representations over infinite layers. Based on this study, we assume that the linear independence holds and build our analysis upon that. This avoids restating the technical assumptions of \cite{daneshmand2020batch} and also further technical refinements of their theorems. The next assumption presents the formal statement of the linear independence property.
 \begin{assumptioncost}{1}{\alpha,\ell}
 \label{assume:lineary_indepdence}
 There exists an absolute positive constant $\alpha$ such that the minimum singular value of $H_k$ is greater than (or equal to) $\alpha$  for all $k=1,\dots, \ell$.
\end{assumptioncost}
 
The linear independence of the representations is a shared property across all layers. However, the representations constantly change when passing through random layers. In this paper, we mathematically characterize the dynamics of the representations across layers. 

\section{Orthogonality of deep representations} \label{sec:orthogonality}
\subsection{A warm-up observation}
To illustrate the difference between BN and vanilla networks, we compare hidden representations of two input samples across the layers of these networks. Figure~\ref{fig:orthogonality} plots the absolute value of cosine similarity of these samples across layers. This plot shows a stark contrast between vanilla and BN networks: While representations become increasingly orthogonal across layers of a BN network, they become increasingly aligned in a vanilla network. We used highly correlated inputs for the BN network and perfectly orthogonal inputs for the Vanilla network to contrast their differences. While the hidden representation of vanilla networks has been theoretically studied \citep{daneshmand2020batch,bjorck2018understanding,saxe2013exact}, up to the best of our knowledge, there is no theoretical analysis of BN networks. In the following section, we formalize and prove this orthogonalizing property for BN networks.  
\begin{figure}[h!]
    \centering
    \begin{tabular}{c}
        \includegraphics[width=0.4\textwidth]{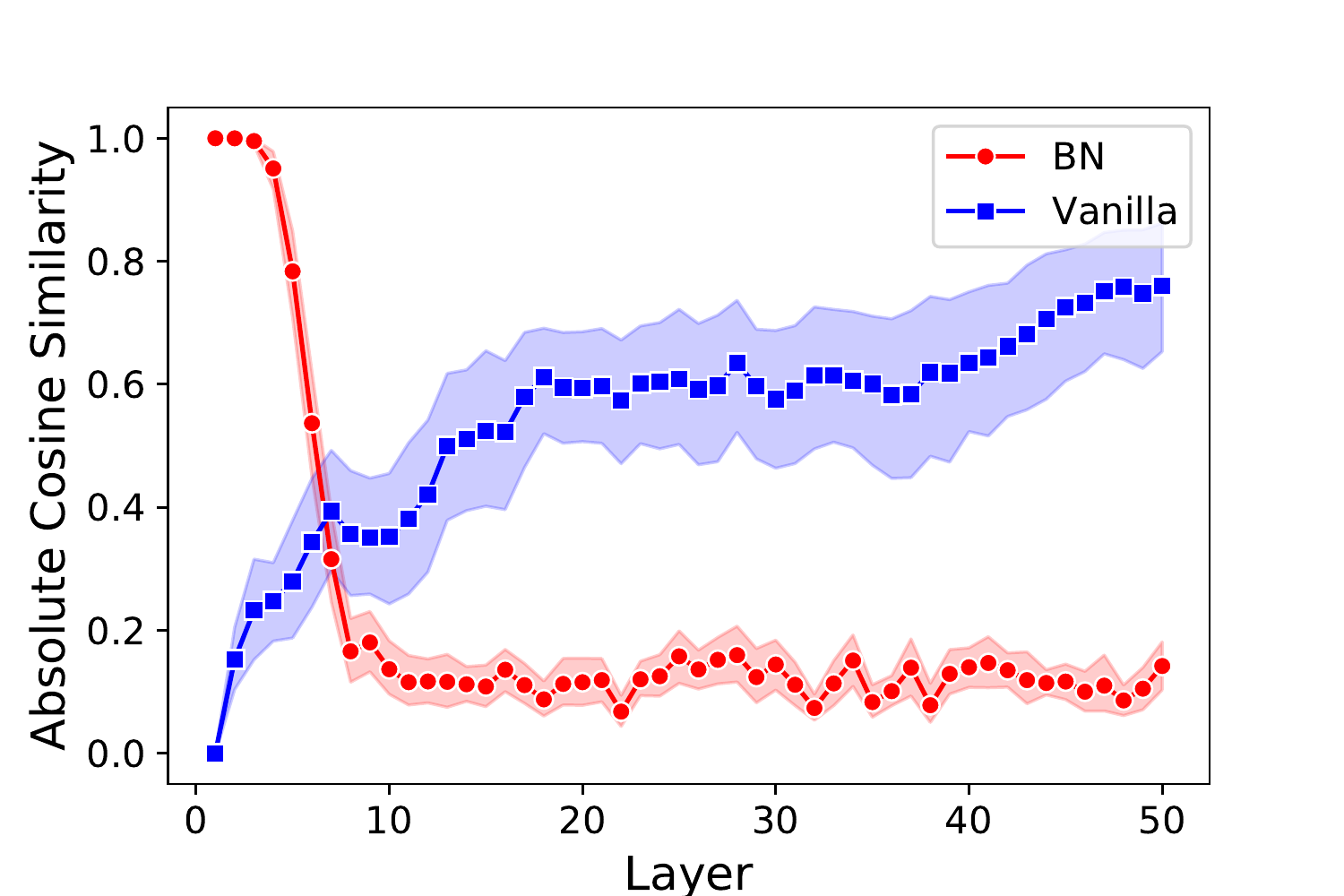}
    \end{tabular}
    \caption{\textbf{Orthogonality: BN vs. vanilla networks.} The horizontal axis shows the number of layers, and the vertical axis shows the absolute value of cosine similarity between two samples across the layers ($d=32$). Mean and 95\% confidence intervals of 20 independent runs.
    }
    \label{fig:orthogonality}
\end{figure}

\subsection{Theoretical analysis}
The notion of orthogonality gap plays a central role in our analysis. Given the hidden representation $H \in \R^{d\times n}$, matrix $H^\top H$ constitutes inner products between representations of different samples. Note that $H^\top H \in \R^{n \times n}$ is different form the covariance matrix $H H^\top/n \in \R^{d\times d}$. The orthogonality gap of $H$ is defined as the deviations of $H^\top H$ from identity matrix, after proper scaling. More precisely, define $V: \R^{d\times n}\setminus\mathbf{0} \to \R_+$ as 
\begin{align}
    V(H) := \Big\| \left(\frac{1}{\|H\|_F^2}\right)H^\top H - \left(\frac{1}{\|I_n\|_F^2}\right)I_n \Big\|_F.
\end{align}
The following theorem establishes a bound on the orthogonality of representation in layer $\ell$.
\begin{theorem} \label{thm:contraction}
Under Assumption \ref{assume:lineary_indepdence}$(\alpha,\ell)$, the following holds:
    \begin{align}\label{eq:exponential_contraction}
        \E \left[ V(H_{\ell+1}) \right] \leq 2 \left(1-\frac{2}{3}\alpha\right)^{\ell}+ \frac{3 n}{\alpha\sqrt{d}}.
    \end{align}
\end{theorem}
\noindent Assumption \ref{assume:lineary_indepdence} is studied by   \cite{daneshmand2020batch} who note that \ref{assume:lineary_indepdence}$(\alpha,\infty)$ holds as long as $d=\Omega(n^2)$. If \ref{assume:lineary_indepdence} does not hold, one can still prove that there is a function of representations that decays with depth up to a constant (see Appendix~\ref{sec:proof_thm}).

The above result implies that BN is an approximation algorithm for orthogonalizing the hidden representations. If we replace $\diag(M)^{-1}$  by $(M)^{-1}$ in BN formula, in Eq.~\eqref{eq:chain_linear}, then all the hidden representation will become exactly orthogonal.  However, computing the inverse of a \textit{non-diagonal} $d\times d$ matrix is computationally expensive, which must repeat for all layers throughout training, and differentiation must propagate back through this inversion. The diagonal approximation in BN significantly reduces the computational complexity of the matrix inversion. Since the orthogonality gap decays at an exponential rate with depth, the approximate orthogonality is met after only a few layers. Interestingly, this yields a desirable cost-accuracy trade-off: For a larger width, the orthogonality is more accurate, due to term $1/\sqrt{d}$ in the orthogonality gap, and also the computational gain is more significant.

From a different angle, Theorem~\ref{thm:contraction} proves the stochastic stability of the Markov chain of hidden representations. In expectation, stochastic processes may obey an inherent Lyapunov-type of stability  \citep{kushner1967stochastic,kushner2003stochastic,khasminskii2011stochastic}. One can analyze the mixing and hitting times of Markov chains based on the stochastic stability of Markov chains,\citep{kemeny1976markov,eberle2009markov}.  
In our analysis, the orthogonality gap is a Lyapunov function characterizing the stability of the chain of hidden representations. This stability opens the door to more theoretical analysis of this chain, such as studying mixing and hitting times. Since these theoretical properties may not be of interest to the machine learning community, we focus on the implications of these results for understanding the underpinnings of neural networks.

It is helpful to compare the orthogonality gap in BN networks to studies on vanilla networks \citep{bjorck2018understanding,daneshmand2020batch,saxe2013exact}. The function implemented by a vanilla linear network is a linear transformation as the product of random weight matrices. The spectral properties of the product of i.i.d.~random matrices are the subject of extensive studies in probability theory~\citep{bougerol2012products}. Invoking these results, one can readily check that the orthogonality gap of the hidden representations in vanilla networks rapidly increases since the representations converge to a rank one matrix after a proper scaling \citep{daneshmand2020batch}.

\subsection{Experimental validations}
Our experiments presented in Fig.~\ref{fig:spectral_contraction_1} validate the  exponential decay rate of $V$ with depth. In this plot, we see that $\log(V_\ell)$ linearly decreases for $\ell=1, \dots, 20$, then it wiggles around a small constant.  Our experiments in Fig.~\ref{fig:spectral_contraction_2} suggest that the $\bigo(1/\sqrt{d})$ dependency on width is  almost tight. Since $V(H_\ell)$ rapidly converges to a ball, the average of $V(H_\ell)$ over layers estimates the radius of this ball. This plot shows how the average of $V(H_\ell)$, over 500 layers, changes with the network width. Since the scale is logarithmic in this plot, slope is close to $-1/2$ in logarithmic scale, validating the $\bigo(1/\sqrt{d})$ dependency implied by Theorem~\ref{thm:contraction}. 

 \begin{figure}[h!]
     \centering
     \begin{subfigure}[b]{0.4\textwidth}
         \centering
         \includegraphics[width=\textwidth]{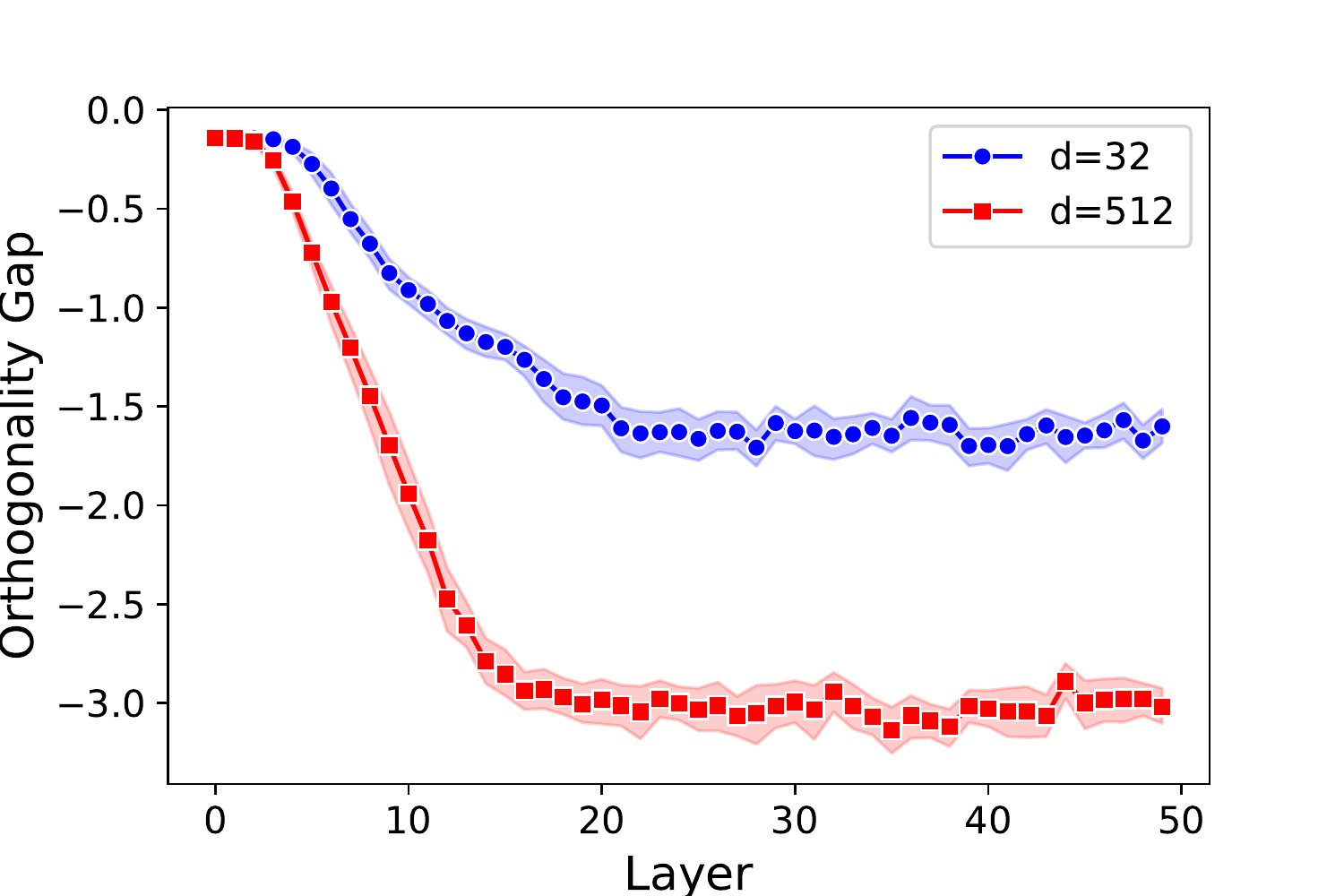}
         \caption{Orthogonality vs. depth}
         \label{fig:spectral_contraction_1}
     \end{subfigure}
     \begin{subfigure}[b]{0.4\textwidth}
         \centering
         \includegraphics[width=\textwidth]{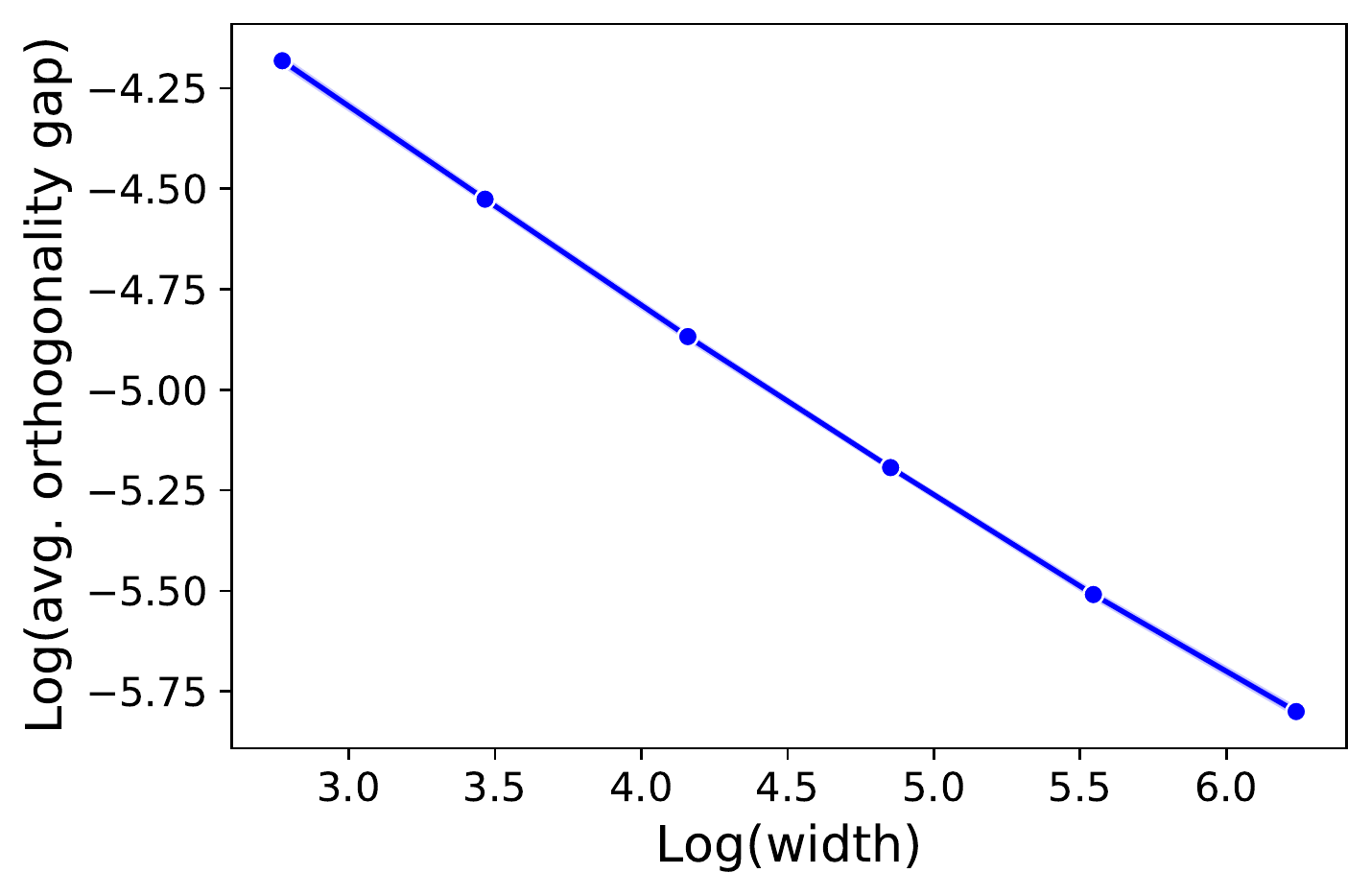}
         \caption{Orthogonality vs. width}
         \label{fig:spectral_contraction_2}
     \end{subfigure}
    \caption{\footnotesize{\textbf{Orthogonality gap vs.\ depth and width.} Left: $\log(V(H_\ell))$ vertically versus $\ell$ horizontally . Right: $\log(\frac{1}{500}\sum_{\ell=1}^{500} V(H_\ell))$ vertically versus $d$ horizontally. The slop of the right plot is $-0.46$. Mean and 95\% confidence interval of 20 independent runs.}  }
\end{figure}

\section{Gaussian approximation} \label{sec:normal_approximation}
\subsection{Orthogonality yields Gaussian approximation}
The established result on the orthogonality gap allows us to show that the representations after linear layers are approximately Gaussian. If $H_\ell$ is orthogonal, the random matrix $W_\ell H_\ell$ is equal in distribution to a standard Gaussian matrix due to the invariance of standard Gaussian distribution under linear orthogonal transformations. We can formalize this notion of Gaussianity by bounding the Wasserstein-2 distance, denoted by $\W_2$, between the distribution of $W_\ell H_\ell$ and standard Gaussian distribution. 
The next lemma formally establishes the link between the orthogonality and the distribution of the representations. 
\begin{lemma} \label{lemma:wv_bound}
Given $G \in \R^{d\times n}$ with i.i.d. zero-mean $1/d$-variance Gaussian elements, the following Gaussian approximation holds:
\begin{align}
\label{eq:bn_recurrence}
    \W_2\left( \law( W_\ell H_\ell ), 
    \law(G/\sqrt{n})\right)^2 \leq 2 n \E \left[ V(H_\ell) \right].
\end{align}
\end{lemma}
Combining the above result with Theorem~\ref{thm:contraction} yields the result presented in the next corollary.
\begin{corollary} \label{cor:gaussian_approximation}
For $G \in \R^{d\times n}$ with i.i.d. zero-mean $1/d$-variance Gaussian elements, 
\begin{align}
    \W_2\left( \law(W_\ell H_\ell), \law(G/\sqrt{n})\right)^2 \leq 4n \left(1-\frac{2}{3}\alpha\right)^{\ell}+ \frac{6 n^2}{\alpha\sqrt{d}}
\end{align}
holds under Assumption~\ref{assume:lineary_indepdence}$(\alpha,\ell)$.
\end{corollary}


In other words, the distribution of the representations contracts to a Wasserstein 2 ball around an isotropic Gaussian distribution as the depth grows. The radius of the Wasserstein 2 ball is at most $\bigo(1/\sqrt{\text{width}})$.
As noted in the last section, \ref{assume:lineary_indepdence} is extensively studied by \cite{daneshmand2020batch} where it is shown that \ref{assume:lineary_indepdence}$(\alpha>0,\infty)$ holds as long as $d = \Omega(n^2)$. 

\subsection{Deep neural networks as Gaussian processes}
Leveraging BN, Corollary~\ref{cor:gaussian_approximation} establishes the first \textit{non-asymptotic} Gaussian approximation for deep random neural networks. For vanilla networks, the Gaussianity is guaranteed only in the \textit{asymptotic} regime of infinite width~\citep{garriga2018deep,neal2012bayesian,lee2017deep,neal2012bayesian,hazan2015steps}. Particularly, \cite{matthews2018gaussian} links vanilla networks to Gaussian processes when their width is infinite and grows in successive layers, while our Gaussian approximation holds for networks with finite width across layers. 
Table \ref{tab:normal} briefly compares Gaussian approximation for vanilla and BN networks.

\begin{table}[h!]
    \centering
    \begin{tabular}{l l l l l l}
    \hline
      Network  & Width & Depth & Distribution of Outputs \\
    \hline 
     Vanilla MLP  & infinite  & finite & Converges to Gaussian as  width $\to \infty$ \\ 
    Vanilla Convnet  & infinite &
    finite & Converges to Gaussian as width $\to \infty$  \\
    BN MLP (Cor.~\ref{cor:gaussian_approximation}) &
    (in)finite & (in)finite & In a $\bigo(\text{width}^{-\sfrac{1}{4}})$-$\W_2$ ball around Gaussian\\
    \hline
    
\end{tabular}
    \vspace{0.07cm}
    \caption{\footnotesize{\textbf{Distribution of representations in random vanilla and BN networks.} For the convolutional network, the width refers to the number of channels. Results for Vanilla MLPs and Vanilla convolution networks are establish by \cite{matthews2018gaussian}, and \cite{garriga2018deep}, respectively. Remarkably, Corollary~\ref{cor:gaussian_approximation} holds for MLP with linear activations.}}
    \label{tab:normal}
\end{table}
The link between Gaussian processes and infinite-width neural networks has inspired several studies to rely on Gaussian representations in deep random networks \citep{klambauer2017self,de2018random,pennington2018emergence,schoenholz2016deep,yang2019mean}. Assuming the representations are Gaussian,  \cite{klambauer2017self} designed novel activation functions that improve the optimization performance, \cite{de2018random} studies the sensitivity of random networks, \cite{pennington2018emergence} highlights the spectral universality in random networks, \cite{schoenholz2016deep} studies information propagating through the network layers, and \cite{yang2019mean} studies gradients propagation through the depth.  Indeed, our analysis implies that including BN imposes the Gaussian representations required for these analyses. Although our result is established for linear activations, we conjecture a similar result holds for non-linear MLPs  (see Appendix~\ref{sec:activations}). 



\section{The orthogonality and optimization}
\label{sec:optimization}
In the preceding sections, we elaborated on the theoretical properties of BN networks in controlled settings. 
In this section, we demonstrate the practical applications of our findings. In the first part, we focus on the relationship between depth and orthogonality. Increasing depth drastically slows the training of neural networks with BN. Furthermore, we observe that as depth grows, the training slowdown highly correlates with the orthogonality gap. This observation suggests that SGD needs to orthogonalize deep representations in order to start classification. This intuition leads us to the following question: If orthogonalization is a prerequisite for training, can we save optimization time by starting from orthogonal representations? To test this experimentally, we devised a weight initialization that guarantees orthogonality of representations. Surprisingly, even in a network without BN, our experiments showed that this initialization avoids the training slow down, affirmatively answering the question. 

Throughout the experiments, we use vanilla MLP (without BN) with a width of 800 across all hidden layers, ReLU activation, and used Xavier's method for weights intialization~\citep{glorot2010understanding}. We use SGD with stepsize $0.01$ and batch size $500$ and for training. The learning task is classification with cross entropy loss for CIFAR10 dataset~\citep[][MIT license]{krizhevsky2009learning}. We use PyTorch~\citep[][BSD license]{NEURIPS2019_9015} and Google Colaboratory platform with a single Tesla-P100 GPU with 16GB memory in all the experiments. The reported orthogonality gap is the average of the orthogonality gap of representation in the last layer.

\subsection{Orthogonality correlates with optimization performance}
In the first experiment, we show that the orthogonality of representations at the initialization correlates with optimization speed. For networks with 15, 30, 45, 60, and 75 widths, we register training loss after 30 epochs and compare it with the initial orthogonality gap. Figure~\ref{fig:slowdown_depth} shows the training loss (blue) and the initial orthogonality gap (red) as a function of depth. We observe that representations are more entangled, i.e., orthogonal, when we increase depth, coinciding with the training slowdown.  Intuitively, the slowdown is due to the additional time SGD must spend to orthogonalize the representations before classification. In the second experiment, we validate this intuitive argument by tracking the orthogonality gap during training. Figure~\ref{fig:slowdown_training} plots the orthogonality gap of output and training loss for a network with 20 layers. We observe that SGD updates are iteratively orthogonalizing representations, marked by the reduction in the orthogonality gap.

 \begin{figure}[h!]
     \centering
     \begin{subfigure}[b]{0.4\textwidth}
         \centering
         \includegraphics[width=\textwidth]{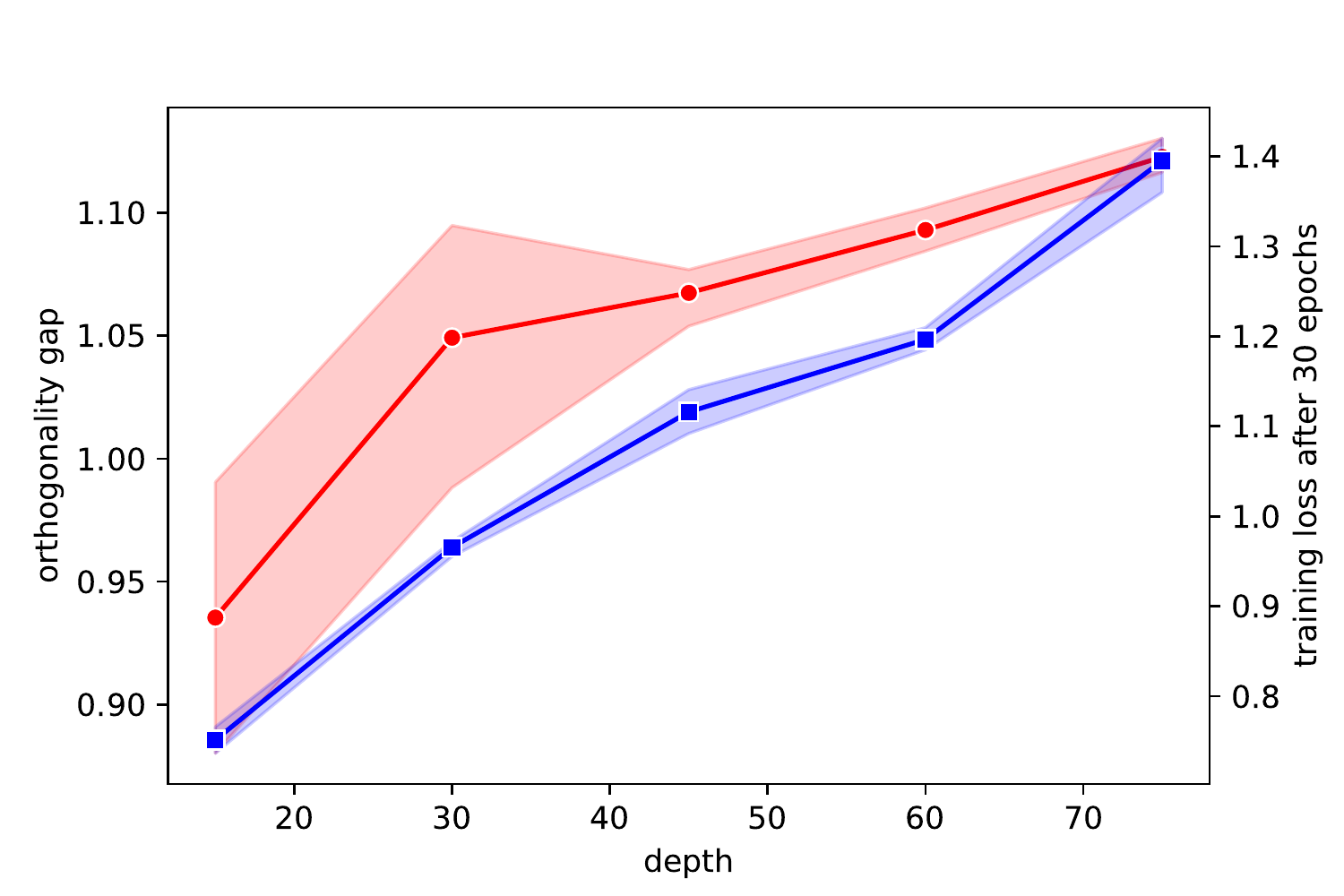}
         \caption{gap and loss vs. depth}
         \label{fig:slowdown_depth}
     \end{subfigure}
     \begin{subfigure}[b]{0.4\textwidth}
         \centering
         \includegraphics[width=\textwidth]{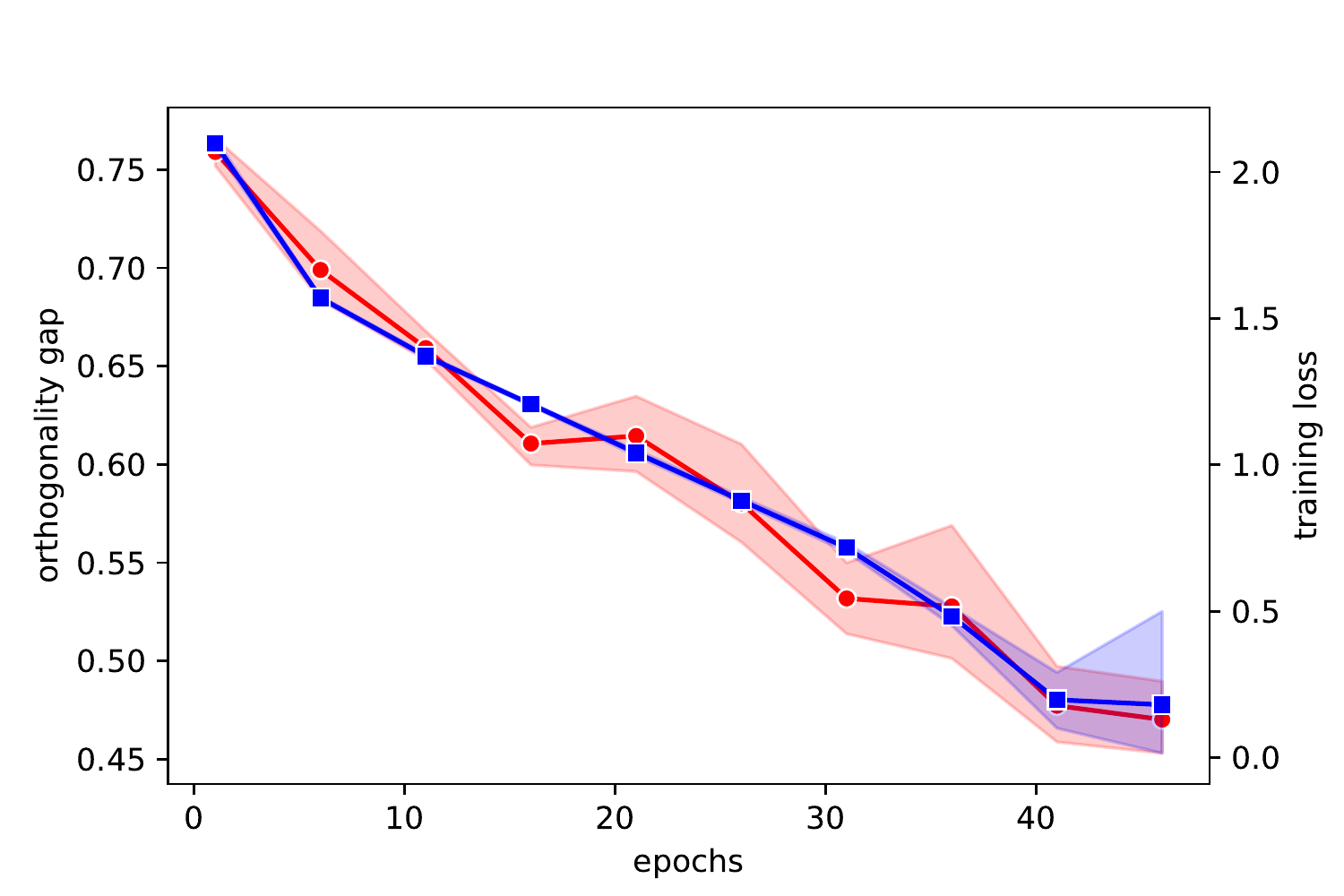}
         \caption{gap and loss during training}
         \label{fig:slowdown_training}
     \end{subfigure}
    \caption{\footnotesize{\textbf{Orthogonality and Optimization} Left: the orthogonality gap at initialization (red, left axis) and the training loss after 30 epochs (blue, right axis) with depth. Right: the orthogonality gap (red, left axis) and  the training loss in each epoch (blue, right axis). Mean and 95\% confidence interval of 4 independent runs. }}
        \label{fig:slowdown_orthogonality}
\end{figure}

\subsection{Learning with initial orthogonal representations}
We have seen that the slowdown in SGD for deeper networks correlates with the orthogonality gap before training.
Here we show that by preemptively orthogonalizing representations, we avoid the slowdown with depth. While in MPL with linear activations, hidden representations remain orthogonal simply by taking orthogonal matrices as weights~\citep{pennington2018emergence,saxe2013exact}, the same does not hold for networks with non-linear activations, such as ReLU. To enforce the orthogonality in the absence of BN, we introduce a dependency between weights of successive layers that ensures deep representations remain orthogonal. More specifically, we incorporate the SVD decomposition of the hidden representation of each layer into the initialization of the subsequent layer. To emphasis this dependency between layers and to distinguish it from purely orthogonal weight initialization, we refer to this as \emph{iterative orthogonalization}.

We take a large batch of samples $n\geq d$, as the input batch for initialization. Let us assume that weights are initialized up to layer $W_{\ell-1}$. To initialize $W_\ell$, we  compute SVD decomposition of the representations $H_\ell = U_\ell \Sigma_\ell V^\top_\ell$ where matrices $U_\ell \in \R^{d \times d}$ and $V_\ell \in \R^{n \times d}$ are orthogonal. Given this decomposition, we initialize $W_\ell$ by 
\begin{align} \label{eq:init_w}
    W_\ell = \frac{1}{\| \Sigma_\ell^{\sfrac{1}{2}}\|_F} V_\ell' \Sigma_\ell^{-\sfrac{1}{2}} U_\ell^\top,
\end{align}
where $V'_\ell \in \R^{d\times d}$ is an orthogonal matrix obtained by slicing $V_\ell \in \R^{n\times d}$. Notably, the inverse in the above formula exists when $n$ is sufficiently larger than $d$ \footnote{We may inductively assume that $H_\ell$ is almost orthogonal by the choice of $W_1, \dots, W_{\ell-1}$. Thus, $\Sigma_\ell$ is invertible. }. 
It is easy to check that  $V(W_\ell H_\ell)< V(H_\ell)$ holds for the above initialization (see Appendix~\ref{sec:init}), similar to BN. 
By enforcing the orthogonality, this initialization significantly alleviates the slow down of training with depth (see Fig.~\ref{fig:cifar_orthogonal}), with no need for BN. This initialization is not limited to MLPs. In Appendix~\ref{sec:experiments}, we propose a similar SVD-based initialization for convolutional networks that effective accelerates training of deep convolutional networks.

 \begin{figure}[h!]
     \centering
         \includegraphics[width=0.55\textwidth]{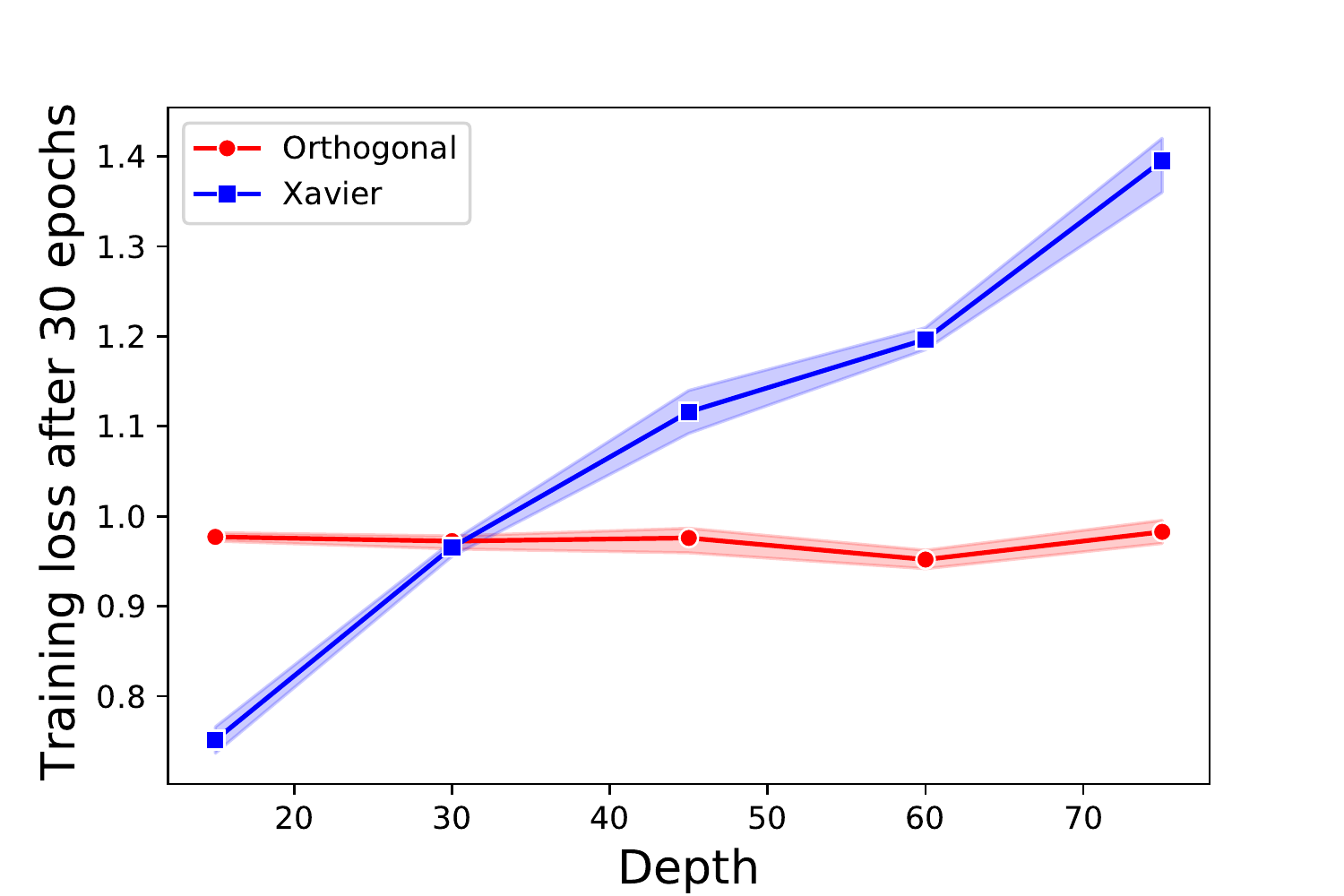}
    \caption{\footnotesize{\textbf{Iterative orthogonalization.} Horizontal axis: depth. Vertical axis: the training loss after 30 epochs for Xavier's initialization (blue), our initialization (red). Mean and 95\% confidence interval of 4 independent runs.}}
        \label{fig:cifar_orthogonal}
\end{figure}

\section{Discussion}
To recap, we proved the recurrence of random linear transformations and BN orthogonalizes samples. Our experiments underline practical applications of this theoretical finding: starting from orthogonal representations effectively avoids the training slowdown with depth for MLPs. Based on our experimental observations in Appendix~\ref{sec:experiments}, we believe that this result extends to standard convolution networks used in practice. In other words, a proper initialization ensuring the orthogonality of hidden representations may replace BN in neural architectures. This future research direction has the potentials to boost the training of deep neural networks and change benchmarks in deep learning.

Although our theoretical bounds hold for MLPs with linear activations, our experiments confirm similar orthogonal stability for various neural architectures. 
Appendix Figure~\ref{fig:relu_convnet} demonstrates this stability for MLPs with ReLU activations and a convolutional networks. This plot compares the evolution of the orthogonality gap, through layers, for BN networks with vanilla networks. For vanilla networks, we observe that the gap increases with depth. On the contrary, the gap decreases by adding BN layers and stabilizes at a term that is constant with regard to depth. Based on our observations, we conjecture that hidden representations of modern BN networks obey similar stability. A more formal statement of the conjecture is presented in Appendix~\ref{sec:activations}. 

\section*{Acknowledgements}
 We thank Gideon Dresdner, Vincent Fortuin, and Ragnar Groot Koerkamp for their helpful comments and discussions. This work was funded in part by the French government under management of Agence Nationale de la Recherche as part of the “Investissements d’avenir” program, reference ANR-19-P3IA-0001(PRAIRIE 3IA Institute). We also acknowledge support from the European Research Council (grant SEQUOIA 724063).




\bibliographystyle{apalike}
\bibliography{refs}

\begin{thebibliography}{}

\bibitem[Arora et~al., 2019a]{arora2018convergence}
Arora, S., Cohen, N., Golowich, N., and Hu, W. (2019a).
\newblock A convergence analysis of gradient descent for deep linear neural
  networks.
\newblock {\em International Conference on Learning Representations}.

\bibitem[Arora et~al., 2019b]{arora2018theoretical}
Arora, S., Li, Z., and Lyu, K. (2019b).
\newblock Theoretical analysis of auto rate-tuning by batch normalization.
\newblock {\em International Conference on Learning Representations}.

\bibitem[Bahri et~al., 2020]{bahri2020statistical}
Bahri, Y., Kadmon, J., Pennington, J., Schoenholz, S.~S., Sohl-Dickstein, J.,
  and Ganguli, S. (2020).
\newblock Statistical mechanics of deep learning.
\newblock {\em Annual Review of Condensed Matter Physics}.

\bibitem[Bartlett et~al., 2019]{bartlett2019gradient}
Bartlett, P.~L., Helmbold, D.~P., and Long, P.~M. (2019).
\newblock Gradient descent with identity initialization efficiently learns
  positive-definite linear transformations by deep residual networks.
\newblock {\em Neural computation}.

\bibitem[Bietti and Mairal, 2019]{bietti2019inductive}
Bietti, A. and Mairal, J. (2019).
\newblock On the inductive bias of neural tangent kernels.
\newblock {\em Advances in Neural Information Processing Systems}.

\bibitem[Bjorck et~al., 2018]{bjorck2018understanding}
Bjorck, J., Gomes, C., Selman, B., and Weinberger, K.~Q. (2018).
\newblock Understanding batch normalization.
\newblock {\em Advances in Neural Information Processing Systems}.

\bibitem[Bougerol et~al., 2012]{bougerol2012products}
Bougerol, P. et~al. (2012).
\newblock {\em Products of random matrices with applications to Schr{\"o}dinger
  operators}.
\newblock Springer Science \& Business Media.

\bibitem[Daneshmand et~al., 2020]{daneshmand2020batch}
Daneshmand, H., Kohler, J., Bach, F., Hofmann, T., and Lucchi, A. (2020).
\newblock Batch normalization provably avoids rank collapse for randomly
  initialised deep networks.
\newblock {\em Advances in Neural Information Processing Systems}.

\bibitem[De~Palma et~al., 2019]{de2018random}
De~Palma, G., Kiani, B.~T., and Lloyd, S. (2019).
\newblock Random deep neural networks are biased towards simple functions.
\newblock {\em Advances in Neural Information Processing Systems}.

\bibitem[Eberle, 2009]{eberle2009markov}
Eberle, A. (2009).
\newblock Markov processes.
\newblock {\em Lecture Notes at University of Bonn}.

\bibitem[Frankle et~al., 2020]{frankle2020training}
Frankle, J., Schwab, D.~J., and Morcos, A.~S. (2020).
\newblock Training batchnorm and only batchnorm: On the expressive power of
  random features in cnns.
\newblock {\em arXiv preprint arXiv:2003.00152}.

\bibitem[Garriga-Alonso et~al., 2019]{garriga2018deep}
Garriga-Alonso, A., Rasmussen, C.~E., and Aitchison, L. (2019).
\newblock Deep convolutional networks as shallow gaussian processes.
\newblock {\em International Conference on Learning Representations}.

\bibitem[Glorot and Bengio, 2010]{glorot2010understanding}
Glorot, X. and Bengio, Y. (2010).
\newblock Understanding the difficulty of training deep feedforward neural
  networks.
\newblock In {\em International Conference on Artificial Intelligence and
  Statistics}.

\bibitem[Hazan and Jaakkola, 2015]{hazan2015steps}
Hazan, T. and Jaakkola, T. (2015).
\newblock Steps toward deep kernel methods from infinite neural networks.
\newblock {\em arXiv preprint arXiv:1508.05133}.

\bibitem[He et~al., 2016]{he2016deep}
He, K., Zhang, X., Ren, S., and Sun, J. (2016).
\newblock Deep residual learning for image recognition.
\newblock In {\em Proceedings of the IEEE conference on computer vision and
  pattern recognition}.

\bibitem[Huang et~al., 2017]{huang2017densely}
Huang, G., Liu, Z., Van Der~Maaten, L., and Weinberger, K.~Q. (2017).
\newblock Densely connected convolutional networks.
\newblock In {\em Proceedings of the IEEE conference on computer vision and
  pattern recognition}.

\bibitem[Huang et~al., 2014]{huang2014kernel}
Huang, P.-S., Avron, H., Sainath, T.~N., Sindhwani, V., and Ramabhadran, B.
  (2014).
\newblock Kernel methods match deep neural networks on timit.
\newblock In {\em ICASSP}.

\bibitem[Ioffe and Szegedy, 2015]{ioffe2015batch}
Ioffe, S. and Szegedy, C. (2015).
\newblock Batch normalization: Accelerating deep network training by reducing
  internal covariate shift.
\newblock In {\em ICML}.

\bibitem[Jacot et~al., 2018]{jacot2018neural}
Jacot, A., Gabriel, F., and Hongler, C. (2018).
\newblock Neural tangent kernel: Convergence and generalization in neural
  networks.
\newblock {\em Advances in Neural Information Processing Systems}.

\bibitem[Karakida et~al., 2019]{karakida2019normalization}
Karakida, R., Akaho, S., and Amari, S.-i. (2019).
\newblock The normalization method for alleviating pathological sharpness in
  wide neural networks.
\newblock In {\em Advances in Neural Information Processing Systems}.

\bibitem[Kemeny and Snell, 1976]{kemeny1976markov}
Kemeny, J.~G. and Snell, J.~L. (1976).
\newblock {\em Markov chains}.
\newblock Springer-Verlag, New York.

\bibitem[Khasminskii, 2011]{khasminskii2011stochastic}
Khasminskii, R. (2011).
\newblock {\em Stochastic stability of differential equations}, volume~66.
\newblock Springer Science \& Business Media.

\bibitem[Klambauer et~al., 2017]{klambauer2017self}
Klambauer, G., Unterthiner, T., Mayr, A., and Hochreiter, S. (2017).
\newblock Self-normalizing neural networks.
\newblock {\em Advances in Neural Information Processing Systems}.

\bibitem[Kohler et~al., 2018]{kohler2018exponential}
Kohler, J., Daneshmand, H., Lucchi, A., Zhou, M., Neymeyr, K., and Hofmann, T.
  (2018).
\newblock Exponential convergence rates for batch normalization: The power of
  length-direction decoupling in non-convex optimization.
\newblock {\em arXiv preprint arXiv:1805.10694}.

\bibitem[Krizhevsky et~al., 2009]{krizhevsky2009learning}
Krizhevsky, A., Hinton, G., et~al. (2009).
\newblock Learning multiple layers of features from tiny images.

\bibitem[Kushner and Yin, 2003]{kushner2003stochastic}
Kushner, H. and Yin, G.~G. (2003).
\newblock {\em Stochastic approximation and recursive algorithms and
  applications}.
\newblock Springer Science \& Business Media.

\bibitem[Kushner, 1967]{kushner1967stochastic}
Kushner, H.~J. (1967).
\newblock Stochastic stability and control.
\newblock Technical report.

\bibitem[Lee et~al., 2019]{lee2017deep}
Lee, J., Bahri, Y., Novak, R., Schoenholz, S.~S., Pennington, J., and
  Sohl-Dickstein, J. (2019).
\newblock Deep neural networks as gaussian processes.
\newblock {\em International Conference on Learning Representations}.

\bibitem[Matthews et~al., 2018]{matthews2018gaussian}
Matthews, A. G. d.~G., Rowland, M., Hron, J., Turner, R.~E., and Ghahramani, Z.
  (2018).
\newblock Gaussian process behaviour in wide deep neural networks.
\newblock {\em arXiv preprint arXiv:1804.11271}.

\bibitem[Neal, 1996]{neal2012bayesian}
Neal, R.~M. (1996).
\newblock {\em Bayesian learning for neural networks}.
\newblock Springer Science \& Business Media.

\bibitem[Paszke et~al., 2019]{NEURIPS2019_9015}
Paszke, A., Gross, S., Massa, F., Lerer, A., Bradbury, J., Chanan, G., Killeen,
  T., Lin, Z., Gimelshein, N., Antiga, L., Desmaison, A., Kopf, A., Yang, E.,
  DeVito, Z., Raison, M., Tejani, A., Chilamkurthy, S., Steiner, B., Fang, L.,
  Bai, J., and Chintala, S. (2019).
\newblock Pytorch: An imperative style, high-performance deep learning library.
\newblock In {\em Advances in Neural Information Processing Systems}.

\bibitem[Pennington et~al., 2018]{pennington2018emergence}
Pennington, J., Schoenholz, S., and Ganguli, S. (2018).
\newblock The emergence of spectral universality in deep networks.
\newblock In {\em International Conference on Artificial Intelligence and
  Statistics}.

\bibitem[Santurkar et~al., 2018]{santurkar2018does}
Santurkar, S., Tsipras, D., Ilyas, A., and Madry, A. (2018).
\newblock How does batch normalization help optimization?(no, it is not about
  internal covariate shift).
\newblock {\em Advances in Neural Information Processing Systems}.

\bibitem[Sawa, 1972]{sawa1972finite}
Sawa, T. (1972).
\newblock Finite-sample properties of the k-class estimators.
\newblock {\em Econometrica: Journal of the Econometric Society}, pages
  653--680.

\bibitem[Saxe et~al., 2014]{saxe2013exact}
Saxe, A.~M., McClelland, J.~L., and Ganguli, S. (2014).
\newblock Exact solutions to the nonlinear dynamics of learning in deep linear
  neural networks.
\newblock {\em International Conference on Learning Representations}.

\bibitem[Schoenholz et~al., 2017]{schoenholz2016deep}
Schoenholz, S.~S., Gilmer, J., Ganguli, S., and Sohl-Dickstein, J. (2017).
\newblock Deep information propagation.
\newblock {\em International Conference on Learning Representations}.

\bibitem[Silver et~al., 2017]{silver2017mastering}
Silver, D., Schrittwieser, J., Simonyan, K., Antonoglou, I., Huang, A., Guez,
  A., Hubert, T., Baker, L., Lai, M., Bolton, A., et~al. (2017).
\newblock Mastering the game of go without human knowledge.
\newblock {\em nature}.

\bibitem[Srivastava et~al., 2014]{srivastava2014dropout}
Srivastava, N., Hinton, G., Krizhevsky, A., Sutskever, I., and Salakhutdinov,
  R. (2014).
\newblock Dropout: a simple way to prevent neural networks from overfitting.
\newblock {\em The Journal of Machine Learning Research}.

\bibitem[Yang et~al., 2019]{yang2019mean}
Yang, G., Pennington, J., Rao, V., Sohl-Dickstein, J., and Schoenholz, S.~S.
  (2019).
\newblock A mean field theory of batch normalization.
\newblock {\em International Conference on Learning Representations}.

\end{thebibliography}
\newpage

\appendix

\begin{center}
   \textbf{\LARGE{Appendix}}
\end{center}
\section{Preliminaries}
\subsection{Notations}
Let $v,w \in \R^{k}$ then $v \odot w \in \R^{n}$ with coordinates
\begin{align}
    [v \odot w ]_{i} = v_i w_i
\end{align}
Furthermore $v^{\otimes 2} \in \R^{k \times k}$ with entities
\begin{align}
    [v^{\otimes 2}]_{ij} = v_i v_j. 
\end{align}
In Table~\ref{tab:notations}, we summarize notations introduced previously. 
\begin{table}[h!]
    \centering
    \begin{tabular}{|l|l l|}
    \hline
       Notation & Type & Definition \\ 
       \hline
       $\ell$  & integer & number of layers\\
       $n$  & integer & batch size \\ 
       $d$ & integer & width of the network\\ 
       $k$ & integer & output dimension \\ 
       $X$ & $\R^{d\times n}$ & input matrix\\ 
       $H_\ell$ & $\R^{d\times n}$ & hidden representations at $\ell$ (obeying Eq.~\eqref{eq:bn_recurrence})\\
       $BN$ & $\R^{d\times n} \to \R^{d\times n}$ & batch normalization layer (defined in Eq.\eqref{eq:bn_recurrence}) \\ 
       $\text{Law}(X)$ &  & the law of random matrix $X$ \\ 
       $\sigma_i(M)$ & $\R^{k_1\times k_2} \to \R_+$ & the $i$th largest singular value of matrix $M$ \\ 
       $I_k$ & $\R^{k \times k}$ & Identity matrix of size $k$ \\
       $\ones_k$ & $\R^n$ &  all-ones vector
      \\ \hline
    \end{tabular}
    \vspace{0.1cm}
    \caption{Notations}
    \label{tab:notations}
\end{table}
\subsection{The Markov chain of hidden representations} Recall the chain of the hidden representation, denoted by $\{ H_\ell\in \R^{d\times n} \}$, obeys the following recurrence: 
\begin{align} \label{eq:BNrecurrence_app}
    H_{\ell+1} = \frac{1}{\sqrt{d}}
    BN(W_\ell H_\ell), \quad BN(M) = \left(\diag(M M^\top)\right)^{-\sfrac{1}{2}}M,
\end{align}
where $W_\ell \in \R^{d\times n}$ are random weight matrices with i.i.d. zero-mean Gaussian elements. It is easy to check that the Frobenius norm of $H_\ell$ is one due to the row-wise normalization: 
\begin{equation}
 \label{eq:unitnorm_app}   
\begin{aligned} 
   \tr \left(  BN(H) BN(H)^\top \right) & = \tr\left( \diag(H H^\top)^{-\sfrac{1}{2}} H H^\top\diag(H H^\top)^{-\sfrac{1}{2}}\right) \\ 
   & = \tr\left( \diag(H H^\top)^{-1} H H^\top\right) \\ 
   & = d.
\end{aligned}
\end{equation}
\subsection{A Lyapunov function for the orthogonal stability} We establish stochastic stability of the  chain of the hidden representations. 
Our analysis is in terms of Lyapunov function $\widehat{V}:\R^{d\times n}$ defined as 
\begin{align}
    \widehat{V}(H) = \frac{1}{n} - (\sigma_n (H))^2,
\end{align}
where $\sigma_n(H)$ is the minimum singular value of matrix $H$. Next lemma proves $\widehat{V}(H_\ell)$ bounds the orthogonality gap. 
\begin{lemma} \label{lemma:Vbound}
For all hidden representations $H_\ell$, the following holds:
\begin{align}
    V(H_\ell) \leq 2 n \widehat{V}(H_\ell) \nonumber. 
\end{align}
\end{lemma}
\begin{proof}
Let $\sigma_1, \dots, \sigma_n$ be the singular values of $H_\ell$. Given these singular values, one can compute $V(H_\ell)$ as 
\begin{align*}
    (V(H_\ell))^2 = \sum_{i=1}^n \left(\sigma_i^2-\frac{1}{n}\right)^2 \nonumber.
\end{align*}
According to Eq.~\eqref{eq:unitnorm_app}, $\sum_{i=1}^2 \sigma_i^2 = 1$ holds. The proof is an immediate consequence of this propery.   
\begin{align} \nonumber
    V^2(H_\ell) &=  \sum_{i=1}^n \sigma_i^4 - 2 \underbrace{\left(\sum_{i=1}^n \sigma_i^2  \right)}_{=1}\frac{1}{n} + \frac{1}{n}\\ 
    & = \sum_{i=1}^n \sigma_i^4 -\frac{1}{n}. \nonumber
\end{align}
Fixing $\sigma_n$, the maximum of $\sum_{i=1}^n \sigma_i^4 -\frac{1}{n}$ subject to $\sum_{i=1}^n \sigma_i^2 = 1$ is met when $\sigma_1^2 = 1-(n-1)\sigma_n^2$ and $\sigma_2 = \dots = \sigma_n$. Therefore,
\begin{align} \nonumber
    V^2(H_\ell) \leq 2(n-1)^2 \underbrace{\left( \frac{1}{n} - \sigma_n^2\right)^2}_{=\widehat{V}^2(H_\ell(X))}
\end{align}
holds true. Taking the square root of both sides concludes the proof.
\end{proof}

\section{Proof of Theorem~\ref{thm:contraction}}

The proof of Theorem~\ref{thm:contraction} relies on the following Theorem that characterizes the change of $\widehat{V}$ in consecutive layers. 
\begin{theorem} \label{thm:single_update}
Sequence $\{ H_\ell \}$ obeys
\begin{align*}
 \E \left[ \widehat{V}(H_{\ell+1}) | H_\ell \right] \leq \left(1- \frac{2}{3} \left(\frac{1}{n}- \widehat{V}(H_\ell)\right) \right) \widehat{V}(H_\ell) + \frac{1}{\sqrt{d}}.
\end{align*}
\end{theorem}
Notably, the above result does not rely on Assumption \ref{assume:lineary_indepdence}.  Assuming that  Assumption \ref{assume:lineary_indepdence} holds, we complete the proof of Theorem~\ref{thm:contraction}. Combining the last Theorem by this assumption, we get 
\begin{align*}
    \E \left[ \widehat{V}(H_{\ell+1}) \right] \leq \left(1-\frac{2}{3}\left(\alpha\right)\right) \E \left[ 
\widehat{V}(H_\ell)\right] + \frac{1}{\sqrt{d}} 
\end{align*}
Induction over $\ell$ yields
\begin{align*}
    \E \left[ \widehat{V}(H_{\ell+1}(X)) \right] & \leq  \left(1-\frac{2}{3}\alpha\right)^{\ell} \E \left[ 
\widehat{V}(H_1)\right] + \left(\sum_{k=1}^\ell (1-\frac{2}{3}\alpha)^k\right)\frac{1}{\sqrt{d}} \\
    & \leq \left(1-\frac{2}{3}\alpha\right)^{\ell} \E \left[ 
\widehat{V}(H_1)\right] + \frac{3}{2\alpha\sqrt{d}} 
\end{align*}
An application of Lemma~\ref{lemma:Vbound} completes the proof: 
\begin{align*}
    \E \left[ V(H_{\ell+1}) \right] & \leq 2 n \E \left[ \widehat{V}(H_{\ell+1}) \right] \\ 
    & \leq 2 \left(1-\frac{2}{3}\alpha\right)^{\ell}+ \frac{3 n}{2 \alpha\sqrt{d}} 
\end{align*}
\subsection{Stability analysis without \ref{assume:lineary_indepdence}}
Using Theorem~\ref{thm:single_update}, we can prove stability of the chain $\{ H_\ell \}$ without Assumption~\ref{assume:lineary_indepdence}. After rearrangement of terms in Theorem~\ref{thm:single_update}, we get
\begin{align*}
    \E \left[ \widehat{V}(H_{\ell+1})| H_\ell \right] - \widehat{V}(H_\ell) \leq - \frac{2}{3} \widehat{V}(H_\ell) \left( \frac{1}{n}- \widehat{V}(H_\ell) \right) + \frac{}{ \sqrt{d}}
\end{align*}
Taking the expectation over $H_\ell$ and average over $\ell$ yields
\begin{align*}
    \E \left[\frac{1}{\ell}\sum_{k=1}^\ell \widehat{V}(H_k) \left( \frac{1}{n}- \widehat{V}(H_k)\right)\right]  \leq \left(\frac{3 \E \left[ \widehat{V}(H_0) \right]}{2 \ell}\right) +\frac{3}{2 \sqrt{d}}
\end{align*}

\section{Proof of Theorem~\ref{thm:single_update}}
\label{sec:proof_thm}
 
\subsection{Spectral decomposition.}
Consider the SVD decomposition of $H_{\ell}$ as $H_{\ell} = U \diag(\sigma) V^\top$ where $U$ and $V$ are orthogonal matrices. Given this decomposition, we get
\begin{align} \label{eq:wh}
    W_{\ell} H_{\ell} = \underbrace{W_{\ell} U}_{W} \diag(\sigma) V^\top 
\end{align}
Since $W_{\ell}$ is Gaussian and $U$ is orthogonal, entities of matrix $W$ are also i.i.d. standard normal. This decomposition will be repeatedly used.

\subsection{Concentration}
Consider matrix $C_{\ell+1} := H_{\ell+1}^\top H_{\ell+1}$ whose eigenvalues are $\sigma_1^2, \dots, \sigma_n^2$. The SVD decomposition of $H_\ell$ allows writing $C_{\ell+1}$ as the average of vectors 
\begin{align*}
    C_{\ell+1} = \frac{1}{d}\sum_{i=1}^d \left( \frac{w_i \odot \sigma }{\| w_i \odot \sigma \|_2 } \right)^{\otimes 2}
\end{align*}
where $w_i \in \R^n$ are rows of matrix $W$ in Eq.\eqref{eq:wh} and $\sigma \in \R^n$ is the vector of singular values of $H_\ell$. Thus, conditioned on $\sigma$, $C_{\ell+1}$ is an empirical average of i.i.d. random vectors. This allows us to prove that this empirical average is concentrated around its expectation. The next lemma states this concentration. 
\begin{lemma} \label{lemma:cov_concentration} 
The following concentration always holds
\begin{align*}
     \E_{W_{\ell}} \| C_{\ell+1} - \E_{W_\ell} \left[ C_{\ell+1}
    \right] \|^2 \leq 1/d  
\end{align*}
where 
\begin{align} \label{eq:pn}
    \E_{W_\ell} \left[ C_{\ell+1}
    \right] = \diag(p_1(\sigma), \dots, p_n(\sigma)), \quad p_i(\sigma) := \E \left[ \frac{\sigma_n^2 w_n^2}{\sum_{k=1}^n \sigma_k^2 w_k^2} \right], \quad w_{i} \stackrel{\text{i.i.d.}}{\sim} \N(0,1)
\end{align}
\end{lemma}
The concentration of $C_{\ell+1}$ allows us to prove that the Lyapunov function $\widehat{V}(H_{\ell+1})$ is concentrated around $1/n-p_n(\sigma)$. 
\begin{lemma}
The following holds 
\label{lemma:concentration}
\begin{align*}
 \E_{W_{\ell}} \left[ \left( \widehat{V}(H_{\ell+1}) - (1/n-p_n(\sigma))  \right)^2 \right] \leq \frac{1}{d}
 \end{align*}.
\end{lemma}
The last lemma allows us to predict the value of random variable $\widehat{V}(H_{\ell+1})$ by deterministic term $1/n-p_n(\sigma)$. 
\subsection{Contraction.} 
The decay in $\widehat{V}(H_{\ell+1})$ with $\ell$ is due to term $1/n-p_n(\sigma)$ in the last lemma. This term is less than (or equal to) $V(H_\ell)$.
\begin{lemma} \label{lemma:concentration_center}
For $p_n(\sigma)$ defined in Eq.~\eqref{eq:pn}, the following holds:
\begin{align*}
  \left( \frac{1}{n}- p_n(\sigma) \right) \leq \left(1- \frac{2}{3} \left(\frac{1}{n}- \widehat{V}(H_\ell)\right) \right) \widehat{V}(H_\ell).
\end{align*}
\end{lemma}
Combining the last lemma by Lemma~\ref{lemma:concentration} concludes the proof of Theorem~\ref{thm:single_update}: 
\begin{align} \label{eq:onestep_result_app}
    \E_{W_\ell} \left[  \widehat{V}(H_{\ell+1}) \right] \leq  \left(1- \frac{2}{3} \left(\frac{1}{n}- V(H_\ell)\right) \right) \widehat{V}(H_\ell) + \frac{1}{\sqrt{d}}.
  \end{align}
To complete the proof, we prove Lemmas~\ref{lemma:cov_concentration}, \ref{lemma:concentration}, and \ref{lemma:concentration_center}.
\subsection{Proof of Lemma~\ref{lemma:cov_concentration}}
Given the spectral decomposition of $H_{\ell}$ in Eq.~\eqref{eq:wh}, we compute element $ij$ of $C_{\ell+1}$, which is denoted by $[C_{\ell+1}]_{ij}$:
\begin{align*}
     [C_{\ell+1}]_{ij} = [A^\top A]_{ij} = \frac{1}{d} \sum_{k=1}^d A_{ki} A_{kj}, \quad A_{ki} = W_{ki} \sigma_i/\sqrt{v_k}
\end{align*}
where $v_k = \sum_{m=1}^n W_{km}^2 \sigma_m^2 $. Since $W_{km}$ are zero mean and unit variace, we get
\begin{align*}
   \E [C_{\ell+1}]_{ij} & = 0 \\
   \E [C_{\ell+1}]_{ij}^2 & = \frac{1}{d^2}\sum_{k=1}^d A_{ki}^2 A_{kj}^2 +  \frac{1}{d^2}\sum_{k,k'}^d \underbrace{\E \left[ A_{ki}^2 A_{kj}^2  A_{k'i}^2 A_{k'j}^2\right]}_{=0} \\ 
   & = \frac{1}{d^2}\sum_{k=1}^d A_{ki}^2 A_{kj}^2 
\end{align*}
holds for $i\neq j$.
By summing up over $i \neq j$, we get 
\begin{align*}
    \sum_{i \neq j} \E [C_{\ell+1}]_{ij} = \frac{1}{d^2} \left(\sum_{k} \left( \sum_{i} A_{ki}^2 \right)^2 - \sum_{ik} A_{ki}^4\right)
\end{align*}
For the diagonal elements, we get 
\begin{align*}
    \E \left[ C_{\ell+1} \right]_{ii} & = \frac{1}{d}\sum_{k=1}^d \E A_{ki}^2 \\ 
    & = \frac{1}{d}\sum_{k=1}^d \E A_{ki}^2 \\ 
    & = \underbrace{\frac{1}{d}\sum_{k=1}^d \E \left[ \frac{W_{ki}^2 \sigma_i^2 }{\sum_{j=1}^n W_{kj}^2 \sigma_j^2} \right]}_{p_i(\sigma)}
\end{align*}
The variance of $[C_{\ell+1}]_{ii}$ is bounded as 
\begin{align*}
     \text{var}([C_{\ell+1}]_{ii}) & = \E \left(\frac{1}{d}\sum_{k=1}^d (A_{ki}^2 - p_i(\sigma))\right)^2 \\ 
     & = \frac{1}{d^2}\sum_{k=1}^d (A_{ki}^2 - p_i(\sigma))^2 \\
     & \leq \frac{1}{d^2}\sum_{k=1}^d A_{ki}^4 
\end{align*}
Combining results for diagonal and off-diagonal elements yields
\begin{align*}
    \E \| C_{\ell+1} - \E_{W_{\ell}} \left[ C_{\ell+1} \right] \|_F^2 & = \sum_{ij} \text{var}([C_{\ell+1}]_{ij}) \\
    & \leq  \frac{1}{d^2} \left( \sum_{i} A_{ki}^2 \right)^2 =1/d
\end{align*}
\subsection{Proof of Lemma~\ref{lemma:concentration}}
Notably, the eigenvalues of $C_{\ell+1}$ are squared singular values of $H_{\ell+1}$. Let $\lambda_n(C)$ denote the $n$th largest eigenvalue of matrix $C$. 
\begin{align*}
    \E_{W_{\ell}} \left[ \left( V(C_{\ell+1}) - \left(\frac{1}{n} - p_n(\sigma)\right) \right)^2 \right] & = \E_{W_{\ell}} \left[ \left(\lambda_n\left(C_{\ell+1}\right) - \lambda_n\left(\E_{W_{\ell}} \left[ C_{\ell+1} \right]\right) \right)^2\right] \\ 
    & \leq \E_{W_\ell} \left[ \| C_{\ell+1} - \E_{W_{\ell}} \left[ C_{\ell+1} \right] \|_F^2  \right] \\ 
    & \leq \frac{1}{d}
\end{align*}
where the last inequality relies on Lemma~\ref{lemma:cov_concentration}. 
\subsection{Proof of Lemma~\ref{lemma:concentration_center}}
The proof is based an application of moment generating function that allows us to compute expectations of ratios of random variables. 
\begin{lemma}[Lemma~1 in ~\citep{sawa1972finite}]
Let $X_1$ be a random variable that is positive with probability one and $X_2$  be an arbitrary random variable. Suppose that there exists a joint moment generating function of $X_1$ and $X_2$: 
\begin{align*}
    \phi(\theta_1,\theta_2) = \E \left[ \exp(\theta_1 X_1 + \theta_2 X_2) \right]
\end{align*}
for $ \theta_1 \leq \epsilon$ and $| \theta_2 | <\epsilon$ where $\epsilon$ is some positive constant. Then 
\begin{align*}
    \E \left[ \frac{X_2}{X_1}\right] = \int_{-\infty}^0 \left[\frac{\partial \phi(\theta_1,\theta_2)}{\partial \theta_2} \right]_{\theta_2 = 0 } d \theta_1 
\end{align*}
\end{lemma}
\noindent To estimate $p_n(\sigma)$, we set $X_2 := \sigma_i^2 w_i^2$ and $X_1 = \sum_{j} \sigma_j^2 w_j^2 $, which obtains 
\begin{align*}
    \phi(\theta_1,\theta_2) & = \E \left[ \exp(\theta_1 X_1 + \theta_2 X_2) \right] \nonumber\\ 
    & = (2\pi)^{-n/2}\int_{-\infty}^{\infty}\exp( (\theta_1 + \theta_2) \sigma_i^2 w_i^2 + \sum_{j\neq i} \theta_1 \sigma_j^2 w_j^2) \exp(-\sum_{k} w_k^2/2) d w \\ 
    & = (2\pi)^{-n/2}\int_{-\infty}^{\infty} \exp(\left(-0.5+(\theta_1+\theta_2) \sigma_i^2\right) w_i^2) d w_i  \left(\prod_{j\neq i} \int_{-\infty}^{\infty} \exp(\left(-0.5+\theta_1 \sigma_j^2\right) w_j^2) d w_j\right)\\ 
    & = \frac{1}{\sqrt{1-2 (\theta_1 +  \theta_2)\sigma_i^2}} \left( \prod_{j\neq i} \frac{1}{\sqrt{1-2\theta_1 \sigma_j^2}} \right).
\end{align*}
Taking derivative with respect to $\theta_2$ yields 
\begin{align*}
    \frac{\partial \phi}{\theta_2}(\theta_1,0) = \frac{\sigma_i^2}{(1-2 \theta_1\sigma_i^2)^{3/2}}\left( \prod_{j\neq i} \frac{1}{\sqrt{1-2\theta_1 \sigma_j^2}} \right)
\end{align*}
Using the result of the last lemma, we get 
\begin{align*}
    p_i(\sigma) = \int_{-\infty}^0 \frac{\sigma_i^2}{(1-2 \theta\sigma_i^2)}\left( \prod_{j} \frac{1}{\sqrt{1-2\theta \sigma_j^2}} \right) d\theta
\end{align*}
Therefore, 
\begin{align} \label{eq:pnbound}
     p_n(\sigma) = \sigma_n^2 f_n(\sigma), \quad f_n(\sigma) :=  \int_{-\infty}^0 \frac{1}{(1-2\theta\sigma_n^2)^{\sfrac{3}{2}}\prod_{j\neq n}(1-2\theta \sigma_j^2)^{\sfrac{1}{2}}} 
 \end{align}
Since $\sum_{i=1}^n \sigma_i^2 = 1$ (see Eq.~\eqref{eq:unitnorm_app}),$f_n(\sigma)$ in the above formulation is minimized when the $\sigma_j^2$s are all equal for all $j\neq n$. Let $\sigma_n^2:=1/n-\delta$ and $\sigma_j^2:=1/n+\delta/(n-1)$ for all $j\neq i$. This allows us to establish a lowerbound on $f_n(\sigma)$ as
\begin{align} \label{eq:g}
f_n(\sigma) &\ge \int_0^\infty \underbrace{\left(1+2\theta(\frac{1}{n}-\delta)\right)^{-\frac{3}{2}}\left(1+2\theta(\frac{1}{n}+\frac{\delta}{n-1})\right)^{-\frac{n-1}{2}}}_{g(\delta):=}
\end{align}
Next lemma proves that $g(\delta)$ is a convex function for $\delta \in [0,1/n]$. 
\begin{lemma} \label{lemma:convexity_g}
  The function $g(\delta)$, which is defined in Eq.~\eqref{eq:g}, is a convex function on domain $\delta \in [0,1/n]$. 
\end{lemma}
The convexity of $g(\delta)$ yields
\begin{align*}
g''(\delta)\ge 0\; \forall \delta &\implies g(\delta) \ge g(0) + \delta g'(0)
\end{align*}
The above bound allow us to bound $f_n(\sigma)$ as 
\begin{align*}
    f_n(\sigma)&\ge \int_0^\infty g(0) + \delta \int_0^\infty 2\theta(1+\frac{2\theta}{n})^{-\frac{n}{2}-2}  \\
 &= 1 + \delta\frac{2n}{n+2}
\end{align*}
Recall $\delta = \frac{1}{n}- \sigma_n^2$.
Combining the above inequality by Eq.~\eqref{eq:pnbound} concludes the proof of the Lemma~\ref{lemma:concentration_center}:
\begin{align*}
    \frac{1}{n} - p_n(\sigma) \leq \left(1-\frac{2n}{(n+2)}\sigma_n^2\right)\left(\frac{1}{n} - \sigma_n^2\right)
\end{align*}
\subsection{Proof of Lemma~\ref{lemma:convexity_g}}
We can show convexity of $g(\delta)$ by showing $g''(\delta)\ge 0$ for all $\delta\in(0,1)$. To this end, 
we define the helper function $h_1(\delta)$ (for the compactness of notations) as
\begin{align*}
h_1(\delta):=(1+2\theta(\frac{1}{n}-\delta))^{-5/2}(1+2\theta(\frac{1}{n}+\frac{\delta}{n-1}))^{-\frac{n-1}{2}-1}.
\end{align*}
Given $h_1$, the derivative of $g$ reads as 
\begin{align*}
g'(\delta) &=  h_1(\delta)\left( (-\frac{3}{2})(-2\theta)\left(1+2\theta(\frac{1}{n}+\frac{\delta}{n-1})\right) - \frac{n-1}{2}\frac{2\theta}{n-1} \left(1+2\theta(\frac{1}{n}-\delta)\right)\right)\\
&= \theta h_1(\delta) \left(3 + \theta\frac{6}{n} + \theta\delta\frac{6}{n-1}-1-\theta\frac{2}{n}+2\theta\delta\right)\\
&= \theta h_1(\delta) \underbrace{\left(2+\theta \frac{4}{n}+\theta\delta\frac{2n+4}{n-1}\right)}_{h_2(\delta):=}\\
&= \theta h_1(\delta) h_2(\delta).
\end{align*}
One can readily check that $h_2'(\delta) \ge 0$. Hence, $h_1'(\delta)\ge 0$ ensures the convexity of $g(\delta)$. Consider the following helper 
$$h_3(\delta):=\left(1+2\theta(\frac{1}{n}-\delta)\right)^{-7/2}\left(1+2\theta(\frac{1}{n}+\frac{\delta}{n-1})\right)^{-\frac{n-1}{2}-2}.$$
Given $h_3$, we compute $h'_1$ as 
\begin{align*}
h_1'(\delta) &= h_3(\delta)\left((-\frac{5}{2})(-2\theta)(1+2\theta(\frac{1}{n}+\frac{\delta}{n-1})) + (-\frac{n-3}{2})\frac{2\theta}{n-1}(1+2\theta(\frac{1}{n}-\delta))\right)\\
&= h_3(\delta)\theta\left(5(1+2\theta/n+\frac{2\theta\delta}{n-1}) - \frac{n-3}{n-1}(1+2\theta/n-2\theta\delta)\right)\\
&=h_3(\delta)\theta\underbrace{\left(\frac{4n-2}{n-1}+2\theta\frac{4n-2}{n-1}+2\theta\delta\frac{n+2}{n-1}\right)}_{h_4(\delta):=}\\
&=\theta h_3(\delta) h_4(\delta).
\end{align*}
Clearly we have $h_3(\delta),h_4(\delta)\ge 0$ for all valid choices of $\delta$ which finishes the proof.

\section{Proof of Lemma~\ref{lemma:wv_bound}}
The main idea is based on a particular coupling of random matrices $W_\ell H_\ell$ and $G$. Consider the truncated SVD decomposition of $H_\ell$ as $H_\ell = U \diag(\sigma) V^\top $ where $U \in \R^{d\times n}$ and $V \in \R^{n \times n}$ are orthogonal matrices. Due to the orthogonality, the law of $W_\ell U$ is the same as those of $G V $. By coupling $W_\ell U = G V$, we get 
\begin{equation*}
\begin{aligned}
    \left(\W_2(\mu(W_\ell H_\ell), \mu(G/\sqrt{n}))\right)^2 & = \inf_{\text{all the couplings}}\E \|   W_\ell U \diag(\sigma) V^\top  - GVV^\top/\sqrt{n} \|_F^2 \\ 
    & \leq \E \| GV \left( \diag(\sigma) - I/\sqrt{n} \right) V^\top \|_F^2 \\ 
    & = \E \tr(GV \left( \diag(\sigma) - I/\sqrt{n} \right) V^\top V \left( \diag(\sigma) - I/\sqrt{n} \right) V^\top G^\top ) \\ 
    & =  \tr(V \left( \diag(\sigma) - I/\sqrt{n} \right) \left( \diag(\sigma) - I/\sqrt{n} \right) V^\top \E \left[ G^\top G \right] ) \\ 
    &= \E \tr( \left( \diag(\sigma) - I/\sqrt{n} \right) \left( \diag(\sigma) - I/\sqrt{n} \right) V^\top V )  \\ 
    &= \E \| \left( \diag(\sigma) - I/\sqrt{n} \right) \|_F^2 \\
    & = \E \left[  \sum_{i=1}^n (\sigma_i -1/\sqrt{n})^2 \right]  \\ 
&= \E \left[  \sum_{i=1}^n \left(\sigma_i^2 -1/n\right)^2/\left(\sigma_i + 1/\sqrt{n}\right)^2 \right] \\
& \leq n \E \left[  \sum_{i=1}^n \left(\sigma_i^2 -1/n\right)^2 \right] \\ 
& \leq n \E \left[ V^2(H_\ell) \right] \\ 
& \leq 2 n \E \left[ V(H_\ell) \right].
\end{aligned}
\end{equation*}

\section{Orthogonality gap for the iterative initialization} \label{sec:init}
Recall the proposed initialization for weights based on SVD decomposition $H_\ell = U_\ell \Sigma_\ell V_\ell^\top$.
\begin{align*}
    W_\ell = \frac{1}{\| \Sigma_\ell^{\sfrac{1}{2}}\|_F} V_\ell' \Sigma_\ell^{-\sfrac{1}{2}} U_\ell^\top.
\end{align*}
Here, we show that 
\begin{align}
     V(H_\ell) > V(W_\ell H_\ell) 
\end{align}
holds as long as $V(H_\ell) \neq 0$ and \ref{assume:lineary_indepdence}$(\alpha,\ell)$ holds. 
Given singular values $H_\ell$, we get 
\begin{equation}
    \label{eq:VHell}
    \begin{aligned} 
    V(H_\ell) &= \sum_{i=1}^n \left( \sigma_i^2 - \frac{1}{n}\right)^2 \\ 
    & = \sum_{i} \sigma_i^4 - 1/n, 
\end{aligned}
\end{equation}
where we used $\sum_{i=1}^n \sigma_i^2 = 1$ (see Eq.~\eqref{sec:perliminaries}).
Now, we compute $V(H_\ell)$ using the singular values. 
\begin{align*}
    W_\ell H_\ell = \frac{1}{\| \Sigma_\ell^{\sfrac{1}{2}}\|_F}  V_\ell' \Sigma^{\sfrac{1}{2}}_\ell V_\ell 
\end{align*}
Hence, the following holds 
\begin{align*}
     V(W_\ell H_\ell) & = \sum_{i=1}^n \left( \frac{\sigma_i}{\sum_j \sigma_j} - \frac{1}{n} \right)^2  \\ 
     & = \frac{1}{(\sum_{i=1}^n \sigma_i)^2} - 1/n
\end{align*}
Combining with Eq.~\eqref{eq:VHell}, we get 
\begin{align*}
    V(H_\ell) - V(W_\ell H_\ell) = \sum_{i=1}^n \sigma_i^4 - \frac{1}{(\sum_{i=1}^n \sigma_i)^2} .
\end{align*}
To show that the right side of the above equation is positive, we need to prove  
\begin{align*}
    \sum_{i} \sigma_i^4 \left(\sum_{i=1}^n \sigma_i\right)^2 > 1 = \left( \sum_{i} \sigma_i^2 \right)^4
\end{align*}
holds. Using Cauchy-Schwarz inequality, we get 
\begin{align*}
    \left( \sum_{i} \sigma_i^2 \right)^4 & = \left( \sum_{i} \sigma_i^{\sfrac{3}{2}} \sigma_i^{\sfrac{1}{2}} \right)^4 \\ 
    & \leq \left( \sum_{i} \sigma_i \right)^2 \left( \sum_{i} \sigma_i^2 \sigma_i \right)^2 \\ 
    & \leq  \left( \sum_{i} \sigma_i \right)^2 \left( \sum_{i} \sigma_i^4 \right)^2,
\end{align*}
where the equality in the above inequality is met only for $\sigma_i^2 = \frac{1}{n}$ (so $V(H_\ell) = 0$) under \ref{assume:lineary_indepdence}.

\section{Activation functions and the orthogonality} \label{sec:activations}
We have shown hidden representations become increasingly orthogonal through layers of BN MLPs with linear activations. 
We observed a similar results for MLP with ReLU activations and also a simple convolution network (see Figure~\ref{fig:relu_convnet}). Here, we elaborate more on the role activation functions in our analysis. 

\begin{figure}[h!]
     \centering
     \begin{subfigure}[b]{0.4\textwidth}
         \centering        \includegraphics[width=\textwidth]{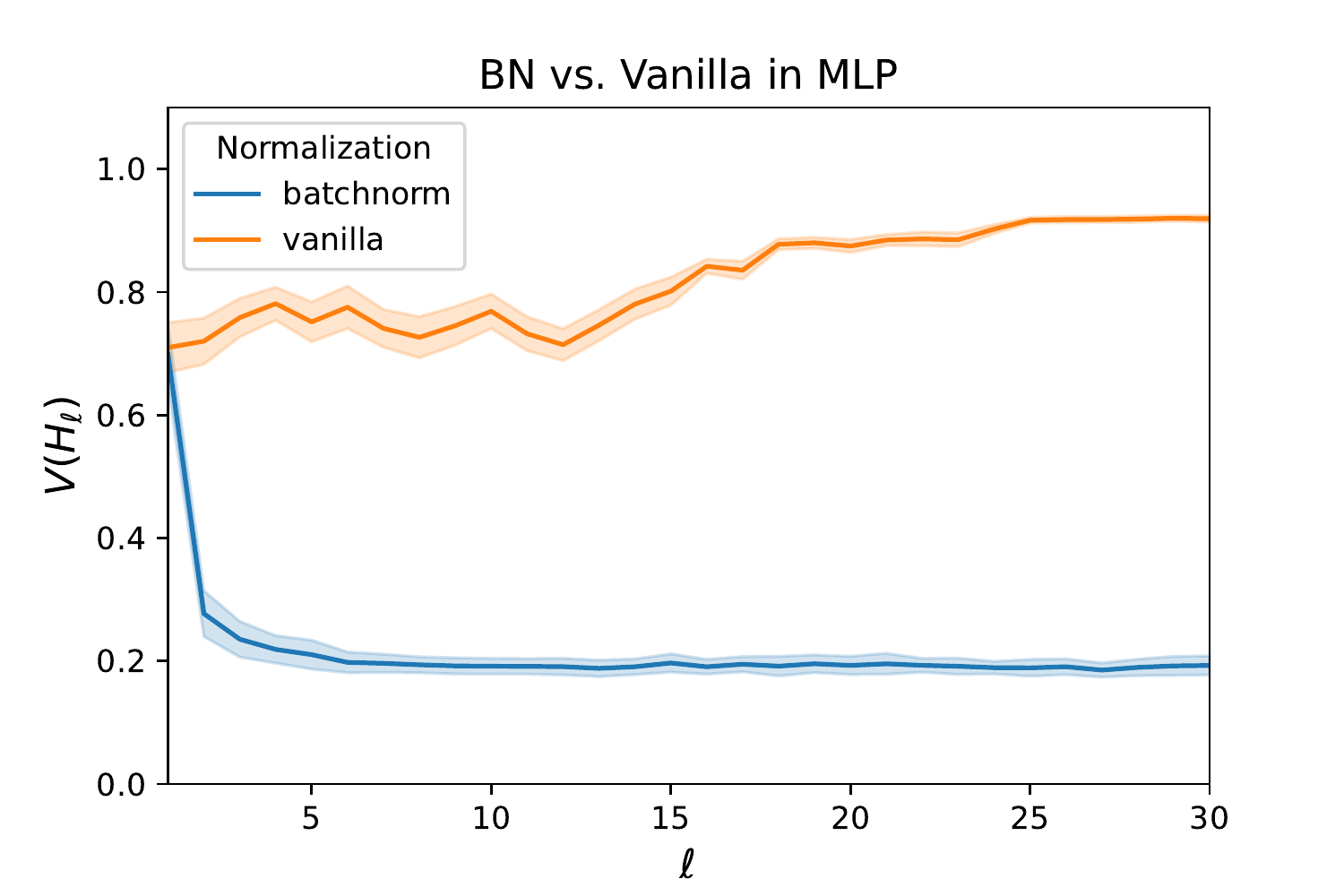}
         \caption{MLP}
         \label{fig:relu_mlp}
     \end{subfigure}
     \begin{subfigure}[b]{0.4\textwidth}
         \centering         \includegraphics[width=\textwidth]{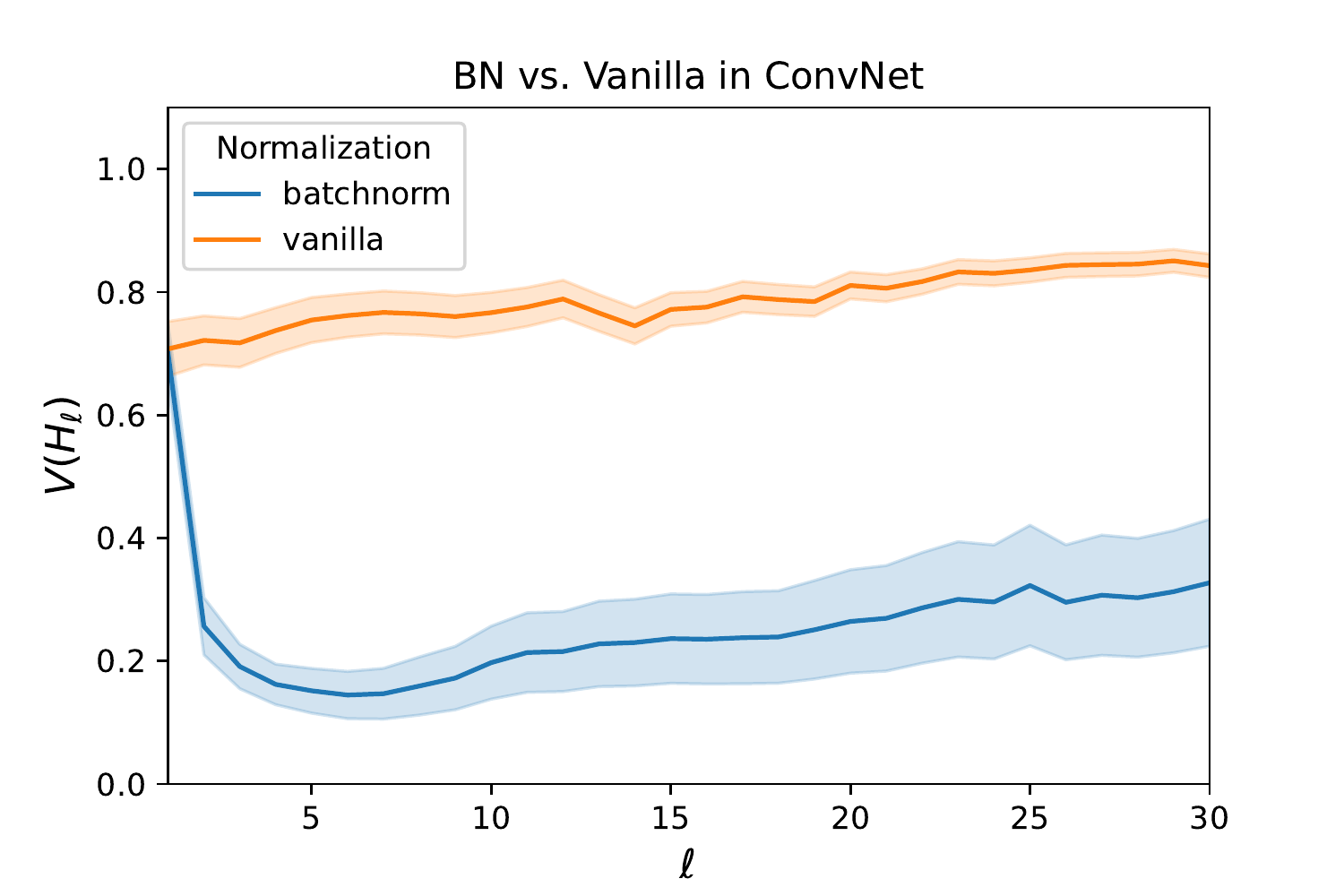}
         \caption{ConvNet}
         \label{fig:convnet}
     \end{subfigure}
        \caption{Orthogonality gap of layers for BatchNorm vs vanilla network with ReLU, for MLP and a basic convolutional network on CIFAR10. Batch size was set to $8$, the number of hidden layers for both networks was set to $30$, with $100$ width for MLP and $100$ channels for ConvNet in all hidden layers. For the ConvNet, the kernel size was set to $3$, across all hidden layers, with zero-padding to make convolutional feature maps equal sized. The standard deviations were computed for $50$ different randomly sampled batches.}
        \label{fig:relu_convnet}
\end{figure}
For an arbitrary activation function $F$, the hidden representations make the following Markov chain: 
\begin{align*}
    Q_{\ell+1} = BN(W_\ell F(Q_\ell)).
\end{align*}
\subsection{Conjecture}
Suppose that $\{ Q_\ell \}$ is \textit{ergodict} \citep{eberle2009markov}, and admits a unique invariant distribution denoted by $\nu$. We define function $L:\R^{d\times n}\to \R_+$ as
\begin{align}
     L(Q_\ell)  = \left\| Q_\ell^\top Q_\ell - \E_{Q\sim \nu} \left[  Q^\top Q \right] \right\|_F .
\end{align}
We conjecture that there exits an integer $k$ such that for $\ell>k$ the following holds: 
\begin{align}
    \E \left[ L(Q_\ell) \right] =\bigo\left( \frac{1}{ \sqrt{d}} \right).
\end{align}
Figure~\ref{fig:activations} experimentally validates that the above result holds for standard activations such as ReLU, sigmoid, and tanh. 

\begin{figure}[h!]
    \centering
    \includegraphics[width=0.45\textwidth]{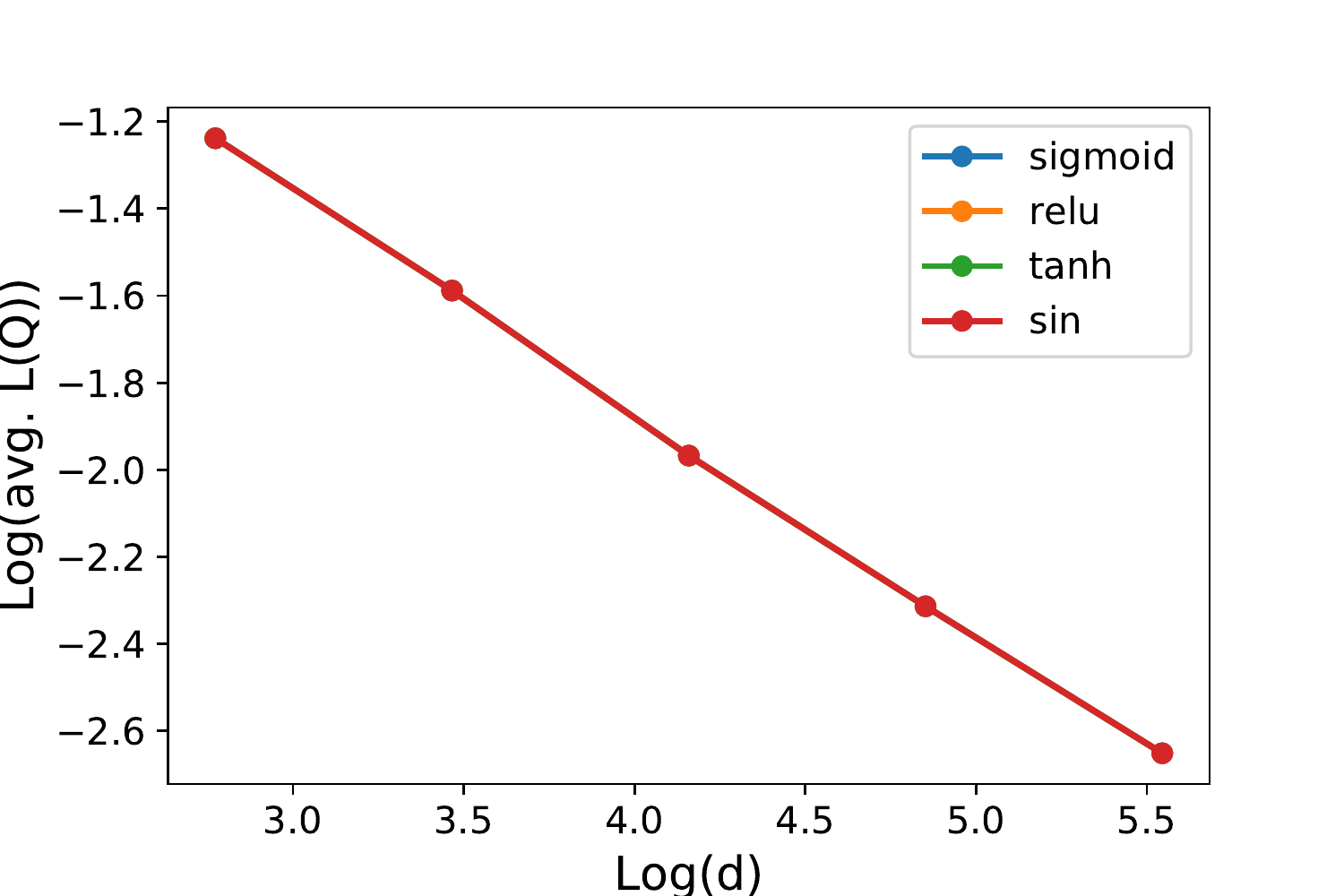}
    \caption{\footnotesize{\textbf{Validating the conjecture.} Horizontal axis: $\log(d)$; vertical axis: $\log(\frac{1}{1000}\sum_{k=1}^{1000} L(Q_k))$. The difference between plots for the activation functions is negligible, so not visible. }}
    \label{fig:activations}
\end{figure}

\subsection{Odd activation}
We observe that Theorem~\ref{thm:contraction} holds for odd activations such as $\sin$ and $\tanh$. For these activation result of Theorem~\ref{thm:contraction}. 
\begin{align*}
    \E \left[ V(H_{\ell+1}) \right]  = \bigo\left( (1-\alpha)^{\ell} + \frac{1}{\sqrt{d}}\right) 
\end{align*}
holds under \ref{assume:lineary_indepdence}$(\alpha,\ell)$. This also aligns with our conjecture for all activations as we observed that $ V = Q$ for these activations.   

\section{Experiments for convolutional networks}
\label{sec:experiments}
The introduced initialization scheme in Section~\ref{sec:optimization} extends to convnets.  In convolutional networks, hidden representation are in the tensor form $H_\ell \in \R^{d\times m \times m \times n}$ where $d$ is the number of filters and $k$ is the image dimensions. Let matrix $W_\ell^{d \times k^2}$ are weights of convolutions layer $\ell$ with kernel size $k$ and $d$ filters.  A 2D convolution is a matrix multiplication combined by projections as:
\begin{align}
    \text{conv2d}(W_\ell,H_\ell) = \text{fold}\left( W_{\ell} \underbrace{\text{unfold}(H_\ell)}_{H'_\ell}\right)
\end{align}
where the unfolding extract batches of $k\times k$ from images in tensor $H_\ell$, folding operation combines the computed convolution for the extracted batches. 
We use the SVD decomposition $H'_\ell = U_\ell \Sigma_\ell V_\ell$ to initialize weights $W_\ell$, exactly the same as Eq.~\eqref{eq:init_w}: 
\begin{align*}
    W_\ell = \frac{1}{\| \Sigma_\ell^{\sfrac{1}{2}}\|_F} V_\ell' \Sigma_\ell^{-\sfrac{1}{2}} U_\ell^\top,
\end{align*}
where $V_\ell'$ is a slice of $V_\ell$.
 We experimentally validate the performance of this initialization for two convolution networks with ReLU activations consist of 20 and 80 convolutions. Table~\ref{tab:convnet} outlines the details of these neural architectures. 
\begin{table}[h!]
    \centering
    \begin{tabular}{|c|c|}
    \hline
       The network with 20 layers  &   The network with 80 layers \\
       \hline
       Conv2d(3, 64)+    MaxPool2d(2) + RELU 
       & Conv2d(3, 64)+   MaxPool2d(2) + RELU  \\
       Conv2d(64, 192) + MaxPool2d(2) + ReLU &  Conv2d(64, 192) + MaxPool2d(2) + ReLU \\
       Conv2d(192, 256)
       + Conv2d(256, 256) +ReLU & Conv2d(192, 256)
        Conv2d(256, 256) + ReLU
        \\ 
        \hline 
        $\{$Conv2d(256, 256)+ReLU$\} \times 16$ &  $\{$Conv2d(256, 256)+ReLU$\} \times 76$
        \\ 
        \hline
        MaxPool2d(2) + ReLU & MaxPool2d(2) + ReLU \\ 
        Dropout(0.5) + Linear(1024,4096) + ReLU & Dropout(0.5) + Linear(1024,4096) + ReLU \\ 
        Dropout(0.5) + Linear(4096,4096) + ReLU & Dropout(0.5) + Linear(4096,4096) + ReLU \\ 
        Linear(4096,10) & Linear(4096,10)
        \\
        \hline
    \end{tabular}
    \caption{\footnotesize{Convolutional networks used in the experiment.} These neural architectures are obtained from adding layers to AlexNet~\citep{srivastava2014dropout}. The kernel size is 3 for all the convolutional layers. }
    \label{tab:convnet}
\end{table}

Figure~\ref{fig:convents} shows the performance of SGD (with stepsize $0.001$ and batchsize 30) for the proposed initialization. This initialization slows SGD for the shallow network with 20 compared to standard Xavier's initialization~\cite{glorot2010understanding}. However, this initialization outperforms Xavier's initialization when the depth is significantly large. This result substantiates the role of the orthogonality in training and shows that the experimental result of Figure~\ref{fig:cifar_orthogonal} extends to convolutional neural networks.      
\begin{figure}
    \centering
    \begin{tabular}{c c}
        20 layers &  80 layers\\
       \includegraphics[width=0.4\textwidth]{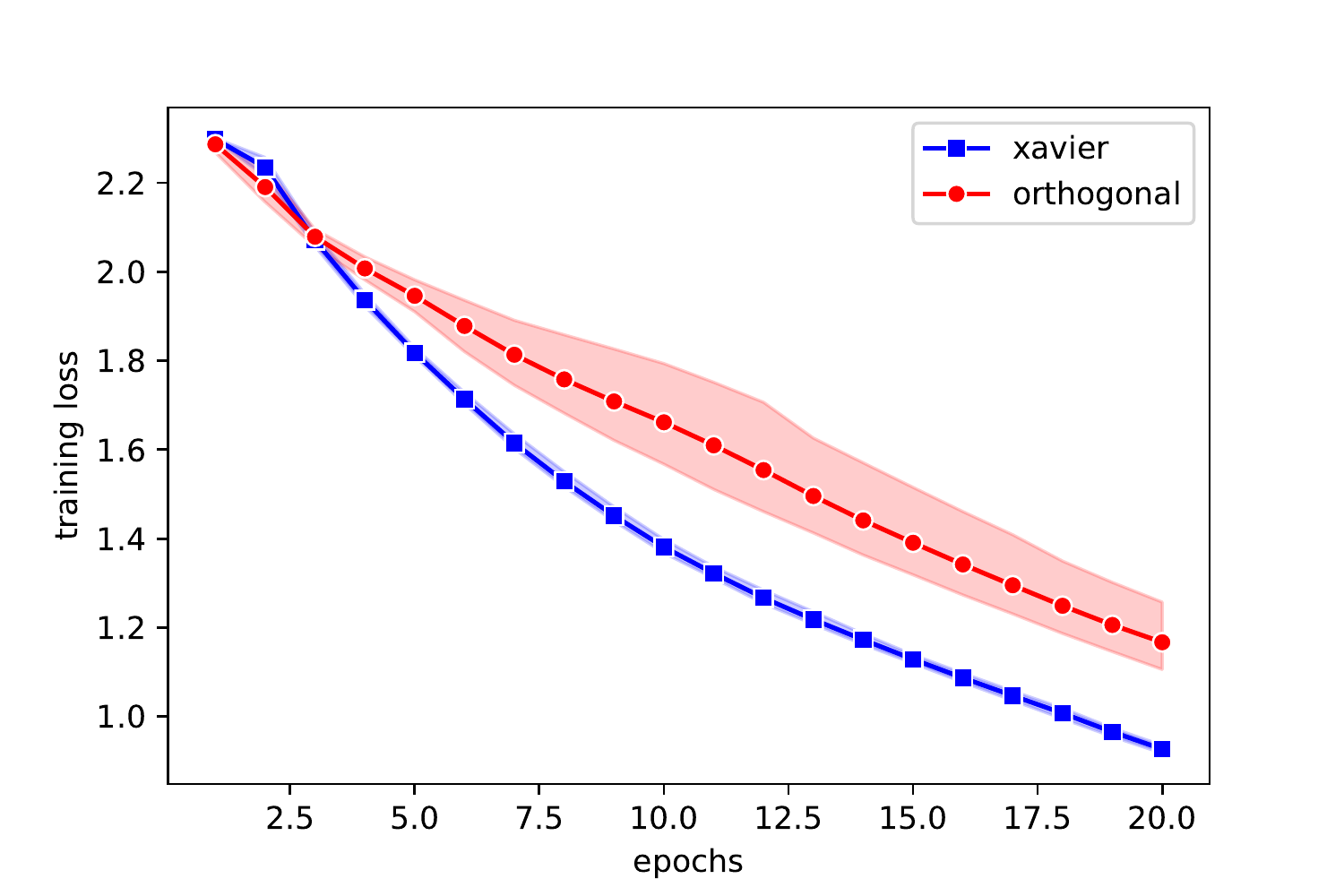}  & \includegraphics[width=0.4\textwidth]{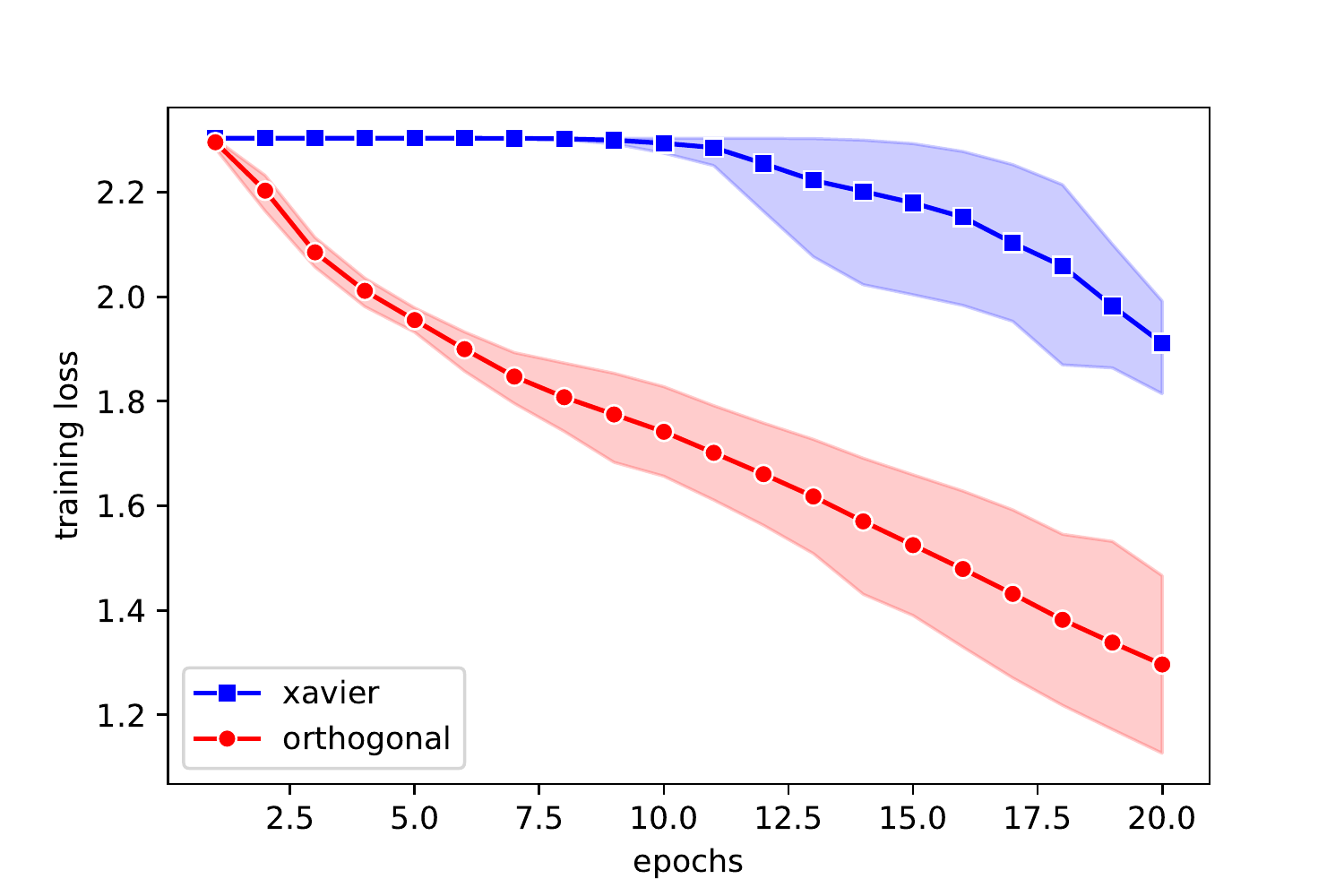}
    \end{tabular}
    \caption{\footnotesize{\textbf{Training convolutional networks.} Convergence of SGD for Xavier initialization \citep{glorot2010understanding} (in blue) and also the proposed initialization method that ensures the orthogonality of hidden representations (in red).  Mean and 95\% confidence interval of 4 independent runs.} }
    \label{fig:convents}
\end{figure}
\newpage

\end{document}